\documentclass[10pt]{article}
\usepackage[margin=1in]{geometry}
\usepackage[T1]{fontenc}
\usepackage{lmodern}
\PassOptionsToPackage{expansion=false}{microtype}
\usepackage{amsmath,amssymb,amsthm,mathtools,microtype,booktabs}
\usepackage[round]{natbib}
\usepackage{tikz}
\usetikzlibrary{arrows.meta,positioning}
\usepackage[hidelinks]{hyperref}
\newtheorem{theorem}{Theorem}[section]
\newtheorem{lemma}[theorem]{Lemma}
\newtheorem{proposition}[theorem]{Proposition}
\newtheorem{corollary}[theorem]{Corollary}
\newcommand{\E}{\mathbb E}
\newcommand{\Prb}{\mathbb P}
\newcommand{\ind}{\mathbf 1}
\newcommand{\OPT}{\mathrm{OPT}}
\newcommand{\ALG}{\mathrm{ALG}}
\newcommand{\LP}{\mathrm{LP}}
\newcommand{\LPB}{\mathrm{LPBud}}
\newcommand{\ph}{\phi}
\newcommand{\rh}{f^\star}
\newcommand{\F}{\mathcal F}
\newcommand{\R}{\mathcal R}
\newcommand{\Var}{\operatorname{Var}}

\begin{document}
\title{Budgeted Bayesian Bandits for Future Utilization}
\author{Sudipto Guha}
\date{}
\maketitle

\begin{abstract}
The multi-armed bandit problem models the trade-off between exploration and exploitation. In this paper, we study the variant where
exploration has a hard budget and precedes exploitation. 
Prior guarantees include a $1/3$ approximation by Guha and Munagala using Lagrangian duality, and a $4/27$ approximation by
Ma against a primal linear programming relaxation for a more general model. We show that the analyses of both algorithms are tight,
even when their relaxations are asymptotically exact. We provide a different efficient algorithm that achieves a $0.43$ approximation. We then
 provide an efficient algorithm whose guarantee approaches $1-e^{-1} \approx 0.63$ for bounded exchangeable reward observations in a non-vanishing regime,
 which is asymptotically tight against natural linear programming relaxations.
 The technique also gives certificates for stopping sequential model evaluation.
 \end{abstract}

\section{Introduction}

Sequential experimental design was
developed in the context of industrial quality control
with roots in Wald's sequential analysis
\citep{wald1945}. \citet{box1951} introduced response surface
methodology, in which each round of experiments chooses factor levels
based on the results so far. Parallel to that optimization-focused approach, sequential data monitoring and
adaptive sampling ideas subsequently became central to the clinical trial literature
\citep{anscombe1963, armitage1969}.
More recently, software-generated alternatives have brought sequential design
back to the fore: there are often many candidate models and only a
limited budget for evaluating them \citep{madani2004,madani2004colt}.
Related problems of measurement and learning arise
in fields such as network monitoring \citep{jung2004,yu2013software}
and consumer studies \citep{caro2007,bertsimas2007}.

Consider $n$ independent, but not necessarily identical, rested
Bayesian arms with known priors and a horizon $T$. A policy $\pi$
plays one arm at each step, observes its outcome, and updates the
posteriors. At the end of the horizon, the policy selects an arm with the
largest posterior utility. For posterior-mean as the utility, let $\theta_i$
denote the latent mean reward of arm $i$, and let $\F^\pi_T$ denote
the information available after exploration. Define the value of the optimum policy $\OPT_T$ as follows
\[
 V_T(\pi)
   = \E\!\left[\max_i \E[\theta_i \mid \F^\pi_T]\right],
 \qquad
 \OPT_T = \sup_\pi V_T(\pi),
\]
where the supremum is over policies satisfying the exploration budget.
Rewards obtained during exploration do not contribute to this objective and standard 
cumulative regret is a very different optimization.
Simple regret \citep{bubeck2009pure,audibert2010best} compares the selected arm with the best arm
under the true means \citep[Section~7.2]{bubeck2012regret}; averaged
over the prior it is $\E[\max_i\theta_i]-V_T(\ALG)$, which splits into
$\E[\max_i\theta_i]-\OPT_T$, the least regret achievable with the
budget, and $\OPT_T-V_T(\ALG)$, the loss from the chosen policy.

Consider $n=2T$ independent,
identically distributed arms, each with $\theta_i=1$ with probability
$1/n$ and $\theta_i=0$ otherwise, where one play reveals $\theta_i$.
This is a needle-in-a-haystack example with
non-vanishing expected reward where the perfect-information value tends to
$1-e^{-1}$, whereas $\OPT_T=1-(1-1/n)^{T+1}\to1-e^{-1/2}$.
The extra term in the exponent accounts for selecting an uninspected
arm after $T$ unsuccessful inspections. The budget thus allows
substantial value, but even the optimal policy has limiting simple regret
$e^{-1/2}-e^{-1}>0$.
If instead $1\ll T=o(n)$ with the same success probability $1/n$,
$\OPT_T\to0$. The optimal policy and a policy inspecting only one
arm then have the same limiting simple regret $1-e^{-1}$, although the
latter obtains only an $O(1/T)$ fraction of $\OPT_T$.
Measuring policies against $\OPT_T$ distinguishes policy quality
from this unavoidable informational loss.

The Gittins index \citep{whittle1983,gittins2011} is optimal for
discounted cumulative reward, but not for a fixed exploration budget
followed by a single terminal choice.
The problem can be formulated as a Markov decision process and solved
by dynamic programming \citep{bellman,blackwell1965} over the joint
product state space of the arms. This computation is exponential in
general, and no exact algorithm is known that runs in time polynomial
in the sum of the single-arm state spaces. \citet{madani2004} show
that the problem is NP-hard when plays of different arms have different
costs, and leave the unit-cost case, studied here, open.
We focus on efficient algorithms guaranteeing
$V_T(\ALG)\geq f \cdot \OPT_T$, where $\ALG$ denotes the algorithm
under consideration (by convention, for maximization problems, the factor $f \leq 1$, larger is better).
Our upper bounds on $\OPT_T$ come from linear programming relaxations
with optimal value $\LP\ge\OPT_T$; the ratio $\LP/\OPT_T$ is the
integrality gap.

The budgeted problem was posed by \citet{madani2004colt}, and
\citet{GM07,GM08} gave the first constant-factor approximations
through LP relaxations.
In their consolidated treatment, \citet[Section~7]{GM13}
(abbreviated GM13 henceforth) obtain a $1/3$ approximation by balancing a Lagrangian bound and executing
the resulting policies. \citet[Section~1.3]{Ma18} gives an LP-based
$4/27$ guarantee for a more general Markovian model that does not
require exchangeability. We show that these guarantees are
tight for the corresponding algorithms, even when their relaxations
are asymptotically exact
(Propositions~\ref{prop:gm13-tight} and~\ref{prop:ma-tight}).
We then provide a better algorithm achieving a factor $(1-e^{-2})/2\approx0.43$
(Section~\ref{sec:231}).

Our main technical tool is selection-preserving truncation for
Bayesian arms with exchangeable observations in $[0,1]$
(Section~\ref{sec:truncation}). The local policies of an optimal LP
solution can be truncated to depth
$k=O(\log(n/(\epsilon\LP))/(\epsilon^2\LP))$,
with no explicit dependence on $T$, while preserving each arm's
selection probability, not increasing its expected play count,
and losing at most $4\epsilon\LP$ reward in aggregate
(Corollary~\ref{lem:depth-k}). The short-policy result of
Section~\ref{sec:short} then incurs little additional loss beyond
the selection-only factor $1-e^{-1}$. Together these give
a $(1-e^{-1}-o(1))$ approximation against $\LP$ when
$T\cdot\LP\gg\log(n/\LP)$, and establish the tight asymptotic
integrality gap $1/(1-e^{-1})$ (Corollary~\ref{cor:asymptotic}).
We call this the \emph{non-vanishing regime}: for $n\ge2$,
$u=\log(n/\LP)/(T \cdot \LP)\to0$, so the LP value does not vanish on the scale
$\log(n/\LP)/T$. The regime allows $\LP\to0$ and $n>T$. The condition implies
that the truncation depth $k$ is $o(T)$, so the scheduling loss vanishes with the parameters set in
Corollary~\ref{cor:asymptotic}.
Outside it, two guarantees
 remain: Theorem~\ref{thm:231} gives $0.43$ on every instance, and for one-play
revelations Theorem~\ref{thm:short} with $k=1$ gives
$(1-e^{-1}-O(\sqrt{\log T/T}))\LP$ even when $\OPT\to0$, as in the needle family  with $T=o(n)$.
Over all exchangeable instances, the worst-case integrality gap of
$\LP$ is therefore at least $1/(1-e^{-1})$ and at most $2/(1-e^{-2})$;
we do not know whether instances outside the regime can exceed
$1/(1-e^{-1})$.

Section~\ref{sec:extensions} uses the dual of the truncated LP as a
computable, instance-specific upper bound on the value of further
evaluation. This gives a stopping rule with a simple-regret guarantee,
and variants for switching, setup and movement budgets.
Section~\ref{sec:terminal-scores} treats terminal utilities, satisficing,
and joint selection and prediction objectives.

\subsection{Related work}
\label{sec:related}

\paragraph{Budgeted learning and model selection.}
Budgeted learning of naive-Bayes classifiers
purchases attribute values of labelled examples under a hard
budget \citep{lizotte2003}. Proposals for this problem include Biased
Robin and single-feature lookahead
\citep{lizotte2003,kapoor2005hard}, reinforcement learning
\citep{kapoor2005rl}, and purchase rules derived from bandit
algorithms \citep{deng2007,deng2013}.
Without a hard budget, the related problem is active
feature-value acquisition \citep{saar2009}.
A bounded active classifier also pays for attributes used
at prediction time \citep{greiner2002}; its budgeted learning
therefore ends with a subset choice
\citep{kapoor2005bounded}. 
Modern model selection likewise allocates limited evaluation
to many candidate configurations \citep{li2018hyperband},
or evaluates pre-trained classifiers under a labelling budget
\citep{karimi2021online}. The latter considers dependence outside our independent-arm
model. 

\paragraph{Optimization formulations.}
Section~\ref{sec:markov-scope} separates selectable
martingale payoffs from the late payoffs in general Markov
models \citep{gupta2011,Ma18}; the latter require
time-indexed relaxations. \citet{goel2009ratio} proposed
a ratio index for budgeted learning.\footnote{Its arXiv
version was withdrawn in 2016 after the authors reported a
bug.}
When one inspection resolves an alternative and inspections
are priced, Pandora's box optimizes expected selected value
minus inspection costs \citep{weitzman1979}. LP relaxations of
adaptive policies are also used for stochastic knapsack
\citep{dean2008} and stochastic probing \citep{GGM10,gupta2013probing}.

\paragraph{Pure exploration, selection, and stopping.}
Bayesian ranking
and selection uses the expected largest posterior mean after
a measurement budget; the knowledge gradient chooses each
measurement by its one-step value \citep{frazier2008}.
Sampling and stopping are already coupled in this literature.
\citet{frazier2008stopping} jointly design them with a
myopic measurement-cost objective.
\citet{eckman2022stopping} give computationally efficient
checks of posterior selection probability and expected
opportunity cost. \citet{hay2012computations} analyze the
value of further sampling, including perfect-information
bounds and finite sampling bounds for one Bernoulli arm
with measurement costs.
In Bayesian optimization, \citet{ishibashi2023stopping}
use changes in an expected simple-regret criterion,
\citet{wilson2024stopping} certifies posterior probability
of near-optimality, and \citet{xie2026stopping} study
cost-adjusted simple regret and expected-budget guarantees.

\paragraph{Resource constraints and switching.}
Bandits with knapsacks add resource budgets to online
reward maximization \citep{badanidiyuru2013knapsacks}.
Their LP bounds and primal-dual resource prices provide an
optimization perspective on learning under limited supply;
\citet{agrawal2014convex} extend the model to concave rewards
and convex constraints. These objectives value outcomes
during learning, whereas here observations precede a
terminal choice. In pure exploration, \citet{li2024resources}
study best-arm identification with resource consumption
on each pull, distinguishing deterministic and stochastic
consumption. Resource use can also depend on the order
of evaluation: \citet{hong2005switching} explicitly trade
sampling against switching in indifference-zone selection,
and \citet{mwai2025switching} constrain the expected
switching rate in fixed-confidence exploration.
For cumulative regret, \citet{simchi2023switching} study
hard switching budgets and characterize phase transitions
in the minimax rate; \citet{koren2017metric} study
adversarial losses with metric movement penalties.

\section{Policies and relaxations}
\label{sec:model}
\label{app:model-details}

There are $n$ independent rested arms. Arm $i$ has root $\rho_i$,
states $u\in\mathcal S_i$, and transitions $u\to v$ with
probabilities $p_{uv}$. Each transition consumes one play. Local
depth is included in the state and is at most $T$. The arm can be
selected at any state for a known nonnegative payoff $r_u$.
Until Section~\ref{sec:markov-scope}, we assume
\begin{equation}\label{eq:super}
 r_u\ \ge\ \sum_{v:u\to v}p_{uv}r_v .
\end{equation}
Equality is the martingale case. Only the terminal selection
contributes to the objective. A global policy uses at most $T$
plays and selects at most one arm.
For the Bayesian model used in Section~\ref{sec:truncation}, each
arm has a fixed latent reward distribution on $[0,1]$, drawn
independently from its prior. Conditional on this distribution,
observations are iid and immediately available. Its mean is
$\theta_i$, and $r_u=\E[\theta_i\mid u]$. Thus observations are
exchangeable, and \eqref{eq:super} holds with equality.

\paragraph{Local policies.}
A single-arm policy $P$ terminates by selecting or quitting, and
uses at most $T$ plays. Following GM13, let
$\widetilde R(P)$ be its expected selected payoff, $I(P)$ its
selection probability, and $W(P)$ its expected play count.
For a fixed policy $P_i$ write
\[
 R_i=\widetilde R(P_i),\quad q_i=I(P_i),\quad w_i=W(P_i).
\]
For each $P_i$, consider its complete local execution from its root
to its own terminal decision, without the shared play budget.
Let $\tau_i$ be its random duration,
$I_i$ its selection indicator, and
$Y_i=r_{u_{\tau_i}}I_i$ its selected payoff, with $Y_i=0$ if it quits.
Hence $R_i=\E Y_i$, $q_i=\E I_i$, and $w_i=\E\tau_i$.
Use independent transitions and local randomization across arms, so
$(\tau_i,I_i,Y_i)$ are independent across $i$.
To analyze a global execution, use these same outcomes and local
decisions whenever it runs $P_i$. The variables above describe the
complete local execution, even if the global execution never starts
$P_i$ or interrupts it before its terminal decision.
The null policy quits immediately without playing or selecting;
its duration, selection indicator, and selected payoff are all zero.

\paragraph{Two constraints.}
The relaxations are standard:
\begin{align}
 &\LP=\max_{\{P_i\}}\left\{\sum_iR_i:
           \sum_iq_i\le1,\ \sum_iw_i\le T\right\},
           \label{eq:two}\\
 &\LPB=\max_{\{P_i\}}\left\{\sum_iR_i:
           \sum_i(q_i+w_i/T)\le2\right\}.
           \label{eq:combined}
\end{align}
The LPBud formulation was introduced in \cite{GM13}. Marginalizing a feasible global policy
over all other arms gives feasible local flows for \eqref{eq:two}.
Adding the two constraints gives
\begin{equation}\label{eq:relaxations}
 \OPT\le\LP\le\LPB.
\end{equation}
For completeness, the local flow variables are reach $z_u$,
play $x_u$, and selection $s_u$, satisfying
\begin{equation}\label{eq:flows}
 z_{\rho_i}=1,\qquad
 z_v=\sum_{u:u\to v}p_{uv}x_u,
 \qquad x_u+s_u\le z_u,\qquad z_u,x_u,s_u\ge0,
\end{equation}
where the equation for $z_v$ applies only to nonroot states
$v\ne\rho_i$, and $x_u=0$ at depth $T$. Then
$R_i=\sum_ur_us_u$, $q_i=\sum_us_u$, and $w_i=\sum_ux_u$.
The unused mass quits. Dividing play and selection flows by
reach probabilities implements the local policy.

\paragraph{Prices and computation.}
With selection price $\eta$ and play price $\lambda$, the
single-arm surplus is
\begin{equation}\label{eq:bellman}
 G_i(u;\eta,\lambda)=\max\bigg\{0,\ r_u-\eta,-\lambda+\sum_v p_{uv}G_i(v;\eta,\lambda)\bigg\}.
\end{equation}
The play option is omitted at depth $T$. Finite LP duality gives
\begin{align}
 \LP&=\min_{\eta,\lambda\ge0}
       \left\{\eta+\lambda T+\sum_iG_i(\rho_i;\eta,\lambda)\right\},
       \label{eq:two-dual}\\
 \LPB&=\min_{\nu\ge0}
       \left\{2\nu+\sum_iG_i(\rho_i;\nu,\nu/T)\right\}.
       \label{eq:one-dual}
\end{align}
Computation is polynomial in the
supplied finite state spaces and rational input encoding.
For succinct or continuous priors, state-space construction or
an equivalent local optimization oracle is an additional
computational assumption.

\subsection{A \texorpdfstring{$(1-e^{-2})/2$}{(1-exp(-2))/2} approximation}
\label{sec:231}

Fix local policies feasible for \eqref{eq:combined}. Discard
zero-reward policies and let
\begin{equation}\label{eq:loads}
 b_i=q_i+w_i/T, \qquad d_i=R_i/b_i,\qquad
 \beta=\sum_i b_i\le2, \qquad \R=\sum_iR_i.
\end{equation}
Order the arms by nonincreasing $d_i$. Their policies run contiguously,
moving to the next when an arm quits.
At a local selection, select that arm
and stop. If the budget interrupts a policy, select the currently
active arm. If the list is exhausted without a selection, choose an
arm with the largest currently selectable payoff $r_{u_i}$.
Make the same choice if the budget is exhausted just as a policy quits.
Selection at a root does not use a play.
\begin{theorem}\label{thm:231}
Let local policies be feasible for \eqref{eq:combined}, with $b_i$, $\beta$
and $\R$ as in \eqref{eq:loads}, and execute them as above. Under
\eqref{eq:super}, $\ALG\ge\ph(\beta)\R$ where $\ph(x)=\frac{1-e^{-x}}x$ with $\ph(0)=1$.
For an optimal solution of \eqref{eq:combined}, this gives $\ALG\ge(1-e^{-2})\LPB/2$.
\end{theorem}

\begin{proof}
Let $A_i$ be the event that policy $i$ starts, and define
$\alpha_i=\Prb(A_i)$. Whether $P_i$ starts depends only on the
outcomes and decisions of earlier policies, so $A_i$ is independent
of $(\tau_i,I_i,Y_i)$ and
$\E[\ind_{A_i}Y_i]=\alpha_iR_i$.

Let $\mathcal H$ be the final execution history, including the
observed outcomes and policy randomization. Write $V_{\rm exec}$
for the selected payoff, so $\ALG=\E[V_{\rm exec}]$.
Every admitted policy except possibly the last has completed.
A policy that quits has $Y_i=0$; a policy that selects has
$Y_i=r_u$, which the execution obtains before stopping.
If the budget interrupts the last policy at state $u$, then
\eqref{eq:super} and bounded optional stopping give
\[
 \E[Y_i\mid\mathcal H]
 \le\E[r_{u_{\tau_i}}\mid\mathcal H]\le r_u.
\]
The first inequality uses $Y_i=0\le r_{u_{\tau_i}}$ when $P_i$
quits. Selecting the current arm obtains $r_u$, which therefore
dominates its expected unfinished payoff. If execution ends after
all started policies have quit, every corresponding $Y_i$ is zero,
and the arm selected by the terminal rule has payoff
$\max_i r_{u_i}\ge0$ by the model assumption.
In each case,
\[
 V_{\rm exec}\ge
 \E\!\left[\sum_i\ind_{A_i}Y_i\mid\mathcal H\right].
\]
Taking expectations gives the required lower bound:
\begin{equation}\label{eq:admitted}
 \ALG\ge\sum_i\E[\ind_{A_i}Y_i]=\sum_i\alpha_iR_i.
\end{equation}
To bound admission probabilities, set $Z_i=I_i+\tau_i/T$, so that
$\E Z_i=b_i$, and define
\[
 H_k=\sum_{i\le k}\alpha_ib_i,\qquad X_k=\sum_{i\le k}b_i,
\]
with $H_0=X_0=0$. Policy $k+1$ is blocked only if execution
has already stopped at a started policy $j\le k$, either by a local
selection ($I_j=1$) or by exhausting the budget. In the latter case,
the complete local durations of the started policies sum to at least
$T$. Thus either reason for blocking implies
$\sum_{i\le k}\ind_{A_i}Z_i\ge1$, and independence of $A_i$ and $Z_i$
shows that this sum has expectation $H_k$.
Denote $(x)^+=\max\{0,x\}$. Markov's inequality gives
\[
    1-\alpha_{k+1}\le H_k,\qquad H_k\ge H_{k-1}+b_k(1-H_{k-1})^+.
 \]
It follows that
\begin{equation}\label{eq:survival}
 H_k\ge1-\prod_{i\le k}(1-b_i)^{+}
    \ge1-e^{-X_k}\ge\ph(\beta)X_k.
\end{equation}
Since $q_i\le1$ and $w_i/T\le1$, a single load $b_i$ can lie in $(1,2]$.
Then $(1-b_i)^+=0$, and the recurrence gives $H_k\ge1$; the positive part
records this case.
Set $d_{n+1}=0$. The density order is used only in the next
step:
\[
 \sum_i \alpha_i R_i =\sum_k(d_k-d_{k+1})H_k
  \ge\ph(\beta)\sum_k(d_k-d_{k+1})X_k
\]
which is $\ph(\beta)\R$.
The theorem follows.
\end{proof}

At a selection or an interruption, choosing instead an arm with the
largest currently selectable payoff (the largest posterior mean in the
Bayesian model) can only increase the payoff and uses no plays, so the
guarantee still holds.

\subsection{Tightness of both relaxations}
\label{sec:both-gap}

The following examples distinguish a loss in the relaxation from
a loss in executing its local policies.

\subsubsection{General martingale signals}\label{app:both-gap-proof}

\begin{proposition}\label{prop:both-gap}
There are instances with independent Markov arms, whose nonnegative payoffs
are martingales selectable at every state, such that $\LPB=\LP=2$, while both
$\OPT$ and $\ALG_{\rm cur}$ (the execution of Theorem~\ref{thm:231}) tend to
$1-e^{-2}$. Thus this execution is asymptotically optimal on these instances,
while both LP gaps approach $2/(1-e^{-2})$.
\end{proposition}

\begin{proof}
Fix $n\ge3$, set $T=n^2$, and define
\[
 p=\frac1{n+1}, \qquad a=\frac1n-\frac pT, \qquad
 z=1-a-p=(1-1/n)^2, \qquad h=1+\frac1{Ta}.
\]
There are $n$ identical independent arms, each with prior mean
$\mu=2/n$. The first play reveals a known value $h$ with
probability $a$, a continuation of posterior mean one with
probability $p$, and a known zero with probability $z$.

On the continuation branch, the latent value is $T$ with
probability $1/T$ and zero otherwise. The next $T-2$ plays
reveal no information. The following play, at local depth $T$,
reveals the latent value. Every intermediate state can be
selected for its posterior mean one. Since $ah+p=2/n$, the
payoffs satisfy the martingale identity at the root as well as
throughout the continuation.\footnote{Dividing all rewards by $T$
rescales to $[0,1]$;
all approximation ratios remain unchanged.}

\paragraph{The LP values.}
For each arm, play once; select the known $h$, quit at zero,
and complete a continuation, selecting only its positive
terminal. The local quantities are
\[
 q=a+p/T=1/n,\qquad R=ah+p=2/n, \qquad
 w=1+p(T-1)=T/n.
\]
Across all arms, selection mass is one, expected work is $T$,
and expected selected reward is two.
These policies are feasible for both LPs and attain value two.
Conversely, nonnegativity and optional stopping imply
$R_i\le r_{\rho_i}$ for every single-arm policy. Hence
\[
 \sum_i R_i\le\sum_i r_{\rho_i}=n(2/n)=2.
\]
Thus $\LPB=\LP=2$.

\paragraph{Analysis.}
All these policies have density
$R/(q+w/T)=1$, so their order is immaterial. The execution
passes through zero branches until it meets an $h$ branch or
a continuation. It selects $h$ immediately in the former case.
If a continuation belongs to the first arm, it can finish and
has expected reward one. If it belongs to a later arm, the
deadline interrupts it before revelation, and selecting its
current state gives reward one. Therefore
\[
 \ALG_{\mathrm{cur}}
   \ge(ah+p)\sum_{j=0}^{n-1}z^j
   =\frac{2n}{2n-1}
        \left[1-(1-1/n)^{2n}\right]
     \longrightarrow1-e^{-2}.
\]
Selecting the largest currently selectable payoff at termination
can only improve this value and requires no additional plays.

\paragraph{Feasible policies.}
Without a completed revelation, only first plays provide
information. Observing all $n$ roots fits within the budget and
maximizes the expected largest selectable payoff among such policies:
further root observations cannot decrease this expectation.
The maximum after observing every root is at most $h$, and is
zero when all roots are zero. Its expectation is therefore at
most $h(1-z^n)$.

A full revelation requires all $T$ plays on one arm, so it can
occur only when the first arm enters continuation, an event of
probability $p$. Conditional on this event, completing the arm
has expected reward at most $1+\mu$: if the revealed value is zero,
an unplayed arm can be selected for its prior payoff $\mu$.
Even crediting this entire contribution
in addition to the root-observation bound gives
\[
 \ALG_{\mathrm{cur}}\le\OPT
       \le h(1-z^n)+p(1+\mu).
\]
Since $h\to1$, $z^n\to e^{-2}$, and $p(1+\mu)\to0$,
the upper bound also tends to $1-e^{-2}$. This proves both
the LP gap and the asymptotic optimality of our execution.
\end{proof}

These arms are martingale signals rather than conditionally iid reward
observations. For exchangeable observations, the next example separates
the two relaxations.

\subsubsection{Exchangeable one-play observations}\label{app:lpbud-gap-proof}

\begin{proposition}\label{prop:lpbud-gap}
There are exchangeable one-play instances with $T=n^2$ for which
\[
 \LPB=2-1/n,\quad \LP=1,\quad
 \OPT=1-(1-1/n)^{2n}.
\]
Thus the LPBud gap still approaches $2/(1-e^{-2})$.
\end{proposition}

\begin{proof}
For $n\ge3$, each arm has latent mean $\theta_i\in\{0,1\}$ with
$p=\Prb(\theta_i=1)=2/n-1/n^2$. Each observation is
Bernoulli$(\theta_i)$, so one play reveals the mean.
Playing every arm once and selecting on a positive observation
gives $q_i=R_i=p$, $w_i=1$, and total combined load two.
The sum of prior means bounds the sum of local selected rewards,
so $\LPB=np=2-1/n$.

For the separate constraints, every selected payoff is at most
one, so $\LP\le1$. Play each arm once and select a positive
observation with probability $1/(np)$; this attains one with
selection mass one and work $n\le T$.
Finally, all arms can be observed, and the optimum equals the
probability that at least one is positive:
$1-(1-p)^n=1-(1-1/n)^{2n}$.
The execution of Theorem~\ref{thm:231} on the policies above selects the first positive arm and obtains
$\OPT$.
\end{proof}

\subsection{GM13's algorithm is tight at \texorpdfstring{$1/3$}{1/3}}
\label{app:gm13-proof}

GM13 balances $\sum_iG_i(\rho_i;\nu,\nu/T)=\nu$, obtaining
the dual bound $3\nu$, and executes in an arbitrary arm order
until the first local selection \citep[Section~7]{GM13}.

\begin{proposition}\label{prop:gm13-tight}
There are exchangeable one-play instances on which GM13's
balanced-price algorithm satisfies $\ALG/\OPT\to1/3$ under
every arm ordering, including density ordering, while
$\LP/\OPT\to1$ and $\LPB/\OPT\to1$.
Every local decision at the balanced price is strict.
\end{proposition}

\begin{proof}
Let $T=n\ge4$. The $n$ arms are independent and identically
distributed, with
\[
 p_n=\frac{2}{n^2},\qquad
 \ell_n=\frac{n-1}{n^2-2},\qquad
 \Prb(\theta_i=1)=p_n,\quad
 \Prb(\theta_i=\ell_n)=1-p_n.
\]
Conditional on $\theta_i$, every reward observation equals
$\theta_i$. These are bounded exchangeable observations, and
one play reveals the mean completely. The prior mean is
$m_n=p_n+(1-p_n)\ell_n=(n+1)/n^2$.
At $\nu=1/n$ we have $0<\ell_n<\nu<m_n<1$.
After a revealing play, the unique optimal local decision is to
select a high arm and quit a low arm; further plays only incur
cost. Thus the root alternatives have surpluses
\[
 \text{quit:}\quad0,\qquad
 \text{select:}\quad m_n-\nu=\frac1{n^2},\qquad
 \text{play:}\quad p_n(1-\nu)-\frac\nu n
                      =\frac{1-2/n}{n^2}.
\]
Immediate selection is strictly optimal, with advantage $2/n^3$
over playing. Since every arm has surplus $1/n^2$, the balance
$\sum_iG_i(\rho_i;\nu,\nu/T)=\nu$ holds exactly. Its solution
is unique because total surplus is nonincreasing in $\nu$ and
the right-hand side is strictly increasing. Every ordering
therefore stops at its first root and obtains $\ALG=m_n$.
Choosing the largest available posterior at this stopping time
has the same value.

\paragraph{Finite price search.}
The binary search in \citet[Lemma~37]{GM13} uses the lower
balancing endpoint, hence a price $0<\widehat\nu\le1/n$.
Immediate selection is strictly optimal throughout this interval.
For $\widehat\nu\le\ell_n$, a revealing play selects either
outcome but pays the extra cost $\widehat\nu/n$.
For $\ell_n<\widehat\nu\le1/n$, the advantage of immediate
selection over playing is
\[
 m_n-p_n-\widehat\nu(1-p_n-1/n)\ge2/n^3>0.
\]
The root selection surplus is positive in both cases.
Consequently the loss does not depend on exact price computation
or on local tie-breaking.

\paragraph{Optimum and relaxation values.}
Observing all $n$ arms is feasible and reveals all latent means,
so it attains the perfect-information upper bound:
\[
 \OPT=1-(1-\ell_n)(1-p_n)^n
      =\frac3n+O(n^{-2}).
\]
The expansion follows from
$\ell_n=1/n+O(n^{-2})$ and
$1-(1-2/n^2)^n=2/n+O(n^{-2})$.
At the balanced price, \eqref{eq:one-dual} gives
\[
 \OPT\le\LP\le\LPB\le3\nu=3/n.
\]
This proves $\LP/\OPT\to1$, $\LPB/\OPT\to1$, and
$\ALG/\OPT\to1/3$.

The balancing equation does not minimize the dual objective.
Indeed, for this instance the immediate-selection and
revealing-play surplus branches meet at
\[
 \nu_{\min}=\frac{n-1}{(n-2)(n+1)}>\frac1n,
 \qquad
 \LPB=\frac{3n+1}{n(n+1)}.
\]
Both branches are positive there. The convex dual objective
has slope $2-n<0$ immediately below this point and
$1-2/n>0$ immediately above it, proving the displayed optimum.
At the balanced price, however, its unique local maximizers
all select immediately: their aggregate combined load is $n$.
Thus even an asymptotically accurate dual bound can be paired
with local policies whose execution loses a factor of three.
No ordering of these policies recovers the missing exploration;
the execution of Section~\ref{sec:231} instead starts from local
policies that satisfy the load constraint.

\paragraph{Distinct prior means.}
Symmetry is not essential. Keep $p_n$ fixed and replace each
low mean by
\[
 \ell_i=\frac{m_n+\eta_i-p_n}{1-p_n},\qquad
 \eta_i=\frac{2i-n-1}{4n^4},\quad i=1,\ldots,n.
\]
The prior means $m_n+\eta_i$ are distinct,
$\sum_i\eta_i=0$, and $|\eta_i|<1/(4n^3)$.
All low means remain in $(0,1/n)$. At $\nu=1/n$ the root
selection surplus is $1/n^2+\eta_i$, whereas the revealing-play
surplus is unchanged. The selection advantage is still positive,
and the total surplus remains $1/n$. The same finite-search
argument applies. Every ordering obtains
$m_n+\eta_i=1/n+O(n^{-2})$ for its first arm, while observing
all arms obtains
$1-(1-\max_i\ell_i)(1-p_n)^n=3/n+O(n^{-2})$.
The dual sandwich and all limiting ratios therefore persist,
including for the now strict density ordering. Selecting the
largest current posterior also leaves the limiting ratio unchanged.
\end{proof}

\subsection{Model boundaries and Ma's rounding}
\label{sec:markov-scope}

The supermartingale condition \eqref{eq:super} lets selection
of an interrupted arm cover its expected remaining selected payoff;
exchangeability is not needed.

\subsubsection{Bounding the selected payoff of an interrupted policy}

Let $\Gamma_i(u)=G_i(u;0,0)$ be the value of the best single-arm policy
from $u$ when plays are free. For a local policy $P$ with terminal
state $u_\tau$, write $Y_P=r_{u_\tau}\ind\{P\text{ selects}\}$.
The following equivalence identifies
the exact property used to bound an interrupted policy's remaining payoff in
Theorem~\ref{thm:231}.

\begin{proposition}\label{prop:late}
For nonnegative payoffs the following are equivalent:
(i) \eqref{eq:super} holds at every state;
(ii) $r_u\ge\E[Y_P\mid u]$ for every state $u$ and every single-arm
policy $P$ started at $u$; (iii) $\Gamma_i(u)=r_u$ at every state.
\end{proposition}
\begin{proof}
(i)$\Rightarrow$(ii) is optional stopping for the nonnegative
supermartingale sequence $r_{u_t}$. (ii)$\Rightarrow$(iii) holds because selecting at
$u$ attains $r_u$, and (iii)$\Rightarrow$(i) because
\eqref{eq:bellman} with $\eta=\lambda=0$ gives
$r_u=\Gamma_i(u)\ge\sum_vp_{uv}\Gamma_i(v)=\sum_vp_{uv}r_v$.
\end{proof}

The broader Markov models of \citet{gupta2011} and \citet{Ma18} allow late
payoffs. For example, take $n=\lfloor T/2\rfloor$ arms whose
first play enters, with probability $1/T$, a chain requiring
$T-1$ more plays to earn one. All earlier payoffs are zero.
Then $\LPB=\LP=n/T$ but $\OPT=1/T$. The time-indexed relaxations
of \citet{gupta2011} and \citet{Ma18} are designed for such late payoffs.

\subsubsection{Ma\textquotesingle s time-indexed rounding}\label{app:ma-tight}

\begin{proposition}\label{prop:ma-tight}
There are three-arm unit-play Markov instances and optimal
solutions of Ma's time-indexed LP for which his $C=3$ rounding,
under every priority tie-breaking rule, has
\[
 \ALG_{\rm Ma}/\OPT\to4/27,\qquad
 \mathrm{LP}_{\rm Ma}/\OPT\to1.
\]
\end{proposition}

\begin{proof}
We use the unit-play rounding of \citet[Sections~4.1--4.2]{Ma18}.
It independently activates each arm with its LP root-play mass
divided by $C=3$, and attaches an LP time label to every subsequent
play. It starts an arm with smallest label and continues until the
next label $t$ satisfies $t\ge2\,\mathrm{depth}(u)$, then chooses
again. We give an instance and an optimal LP solution for which
the loss holds for every priority tie-breaking rule and every
valid local flow decomposition.

\paragraph{Instance and feasible optimum.}
For an integer $m\ge3$, set
\[
 T=2m,\qquad p=1/m,\qquad \delta=m^{-3},\qquad
 q=\frac{m-1}{m+1}.
\]
The valuable arm's first play goes to a dead state with probability
$1-p$, and otherwise begins a deterministic chain of $m$ further
plays whose last play earns one. Each of two blocking arms is a
deterministic chain of $m+1$ plays earning $\delta$ only on its last
play. Arms terminate after earning their reward. Thus any positive
payoff requires $m+1$ plays on its arm, and $T<2(m+1)$ permits
at most one such payoff. Equivalently, give an unfinished arm
selectable payoff zero and a completed arm its earned payoff;
maximizing the terminal choice has the same value.

Testing the valuable arm first, completing it on success, and
otherwise completing one blocker attains
\begin{equation}\label{eq:ma-opt}
 \OPT=p+(1-p)\delta.
\end{equation}
This is optimal: on the valuable arm's good outcome no policy can
earn more than one, and on its bad outcome no policy can earn more
than $\delta$.

\paragraph{A feasible time-indexed solution.}
Play the valuable root at time label one with mass one. On success,
assign its $m$ continuation plays the labels
$m+1,m+2,\ldots,2m$. Its expected work is $1+mp=2$ and
its objective contribution is $p$.

Each blocker has root-play mass $q$ and completes whenever played.
Conditional on activation, choose $m+1$ distinct time labels from
$\{2,\ldots,2m\}$, in increasing order. The marginal probability
of using each of the first $m-1$ labels is $1/(2q)$ and of using
each of the last $m$ labels is $(1-p)/(2q)$. Such a distribution
exists explicitly: choose an integer $J$ supported on the floor
and ceiling of $(m+1)/2$ with that mean; choose $J$ labels uniformly
from the first group and $m+1-J$ uniformly from the second group.
For $m\ge3$, $1\le J\le m-1$ and $1\le m+1-J\le m$.
The required marginals follow from
\[
 q\,\frac{\E J}{m-1}=\frac12,
 \qquad
 q\,\frac{m+1-\E J}{m}=\frac{1-p}{2}.
\]

These increasing label sequences induce feasible single-arm flows
in Ma's time-indexed relaxation. At label one the valuable root
uses all capacity. At each of the next $m-1$ labels the blockers
use $1/2+1/2=1$; at each of the final $m$ labels they use
$(1-p)/2+(1-p)/2$, and the valuable continuation uses $p$.
Every time constraint is therefore met with equality. The value is
\begin{equation}\label{eq:ma-lp}
 \mathrm{LP}_{\rm Ma}=p+2q\delta.
\end{equation}

To prove optimality, let $a$ be the valuable arm's total root-play
mass, $v$ its completion mass, and $s$ the sum of blocker completion
masses in any feasible LP solution. Flow conservation and total
time capacity imply
\[
 v\le pa\le p,\qquad a+mv+(m+1)s\le T.
\]
Its reward is consequently at most
\[
 v+\delta s
 \le \frac{\delta T}{m+1}
       +v\left(1-\frac{\delta m}{m+1}\right)
       -\frac{\delta a}{m+1}
 \le \frac{\delta T}{m+1}
       +v\left(1-\frac{\delta(m+1/p)}{m+1}\right)
 \le p+\frac{\delta(T-1-mp)}{m+1}
  =p+2q\delta.
\]
The coefficient of $v$ is positive, so the last step uses $v\le p$.
Together with the feasible solution, this proves \eqref{eq:ma-lp}.

\paragraph{The rounding loss, independent of ties.}
The valuable arm is activated with probability $1/3$, and each
blocker independently with probability $q/3$. All the flow mass
of an activated blocker reaches its last play: each depth has the
same total LP flow $q$, so any valid flow decomposition abandons
it with probability zero. Its labels increase strictly and are
at most $T$. Consequently, its first $m$ plays have labels
strictly below $T$, under every such decomposition.

Conditional on activation and a good first outcome, the valuable
arm has its next play at depth $j$ and label $m+j$, for
$j=1,\ldots,m$. Since $m+j\ge2j$, the execution returns to the
priority comparison after each play. In particular, its final
play is compared at depth $m$ with label $T=2m$.
Before this final play can be executed, every active blocker
must have made its first $m$ plays: otherwise its next label
is strictly below $T$ and precedes the valuable arm's final play.
Even one active blocker would force at least
\[
 m+(m+1)=2m+1>T
\]
plays before the valuable reward is earned. Therefore the
valuable arm completes if and only if it is activated, its first
outcome is good, and neither blocker is activated. Its contribution
is exactly $p(1-q/3)^2/3$. There is at most one blocker payoff, giving
\[
 \ALG_{\rm Ma}=
 \frac p3\left(1-\frac q3\right)^2+O(\delta).
\]
Since $q\to1$ and $\delta/p\to0$, equations
\eqref{eq:ma-opt}--\eqref{eq:ma-lp} prove both limits in
Proposition~\ref{prop:ma-tight}. More generally, activating each
arm with its root-play mass times a fixed factor $x\in[0,1]$ gives
limiting ratio $x(1-x)^2$, whose maximum is $4/27$ at $x=1/3$.
\end{proof}

Short policies avoid interruption
and therefore also permit late payoffs
(Theorem~\ref{thm:short}).

\section{Short policies}
\label{sec:short}

The construction follows the stochastic-packing execution framework of
\citet{GM07}, also developed in \citet[Sections~3.2--3.3]{GM08}.
Bounded policy durations let us treat the selection constraint and
the play budget separately.

The stronger, standalone selection constraint in $\LP$ can be analyzed once the
deadline is unlikely to interfere with the local policies.
Throughout this section, fix independent local policies with nonnegative
selected rewards satisfying $\sum_iq_i\le1$ and $\sum_iw_i\le T$, and write
$\R=\sum_iR_i$. Let $\rh=1-e^{-1}$.
Work means the number of plays: $\tau_i$ is the realized play count
and $w_i=\E[\tau_i]$ its expectation. Lemma~\ref{lem:selection-order}
first ignores the deadline and bounds the loss from stopping at the
first selection. Theorem~\ref{thm:short} then restores the hard budget
for policies of duration at most $k$, using random thinning to create
slack and reserving $k$ plays before starting a policy.
Figure~\ref{fig:short-execution} separates these two arguments.

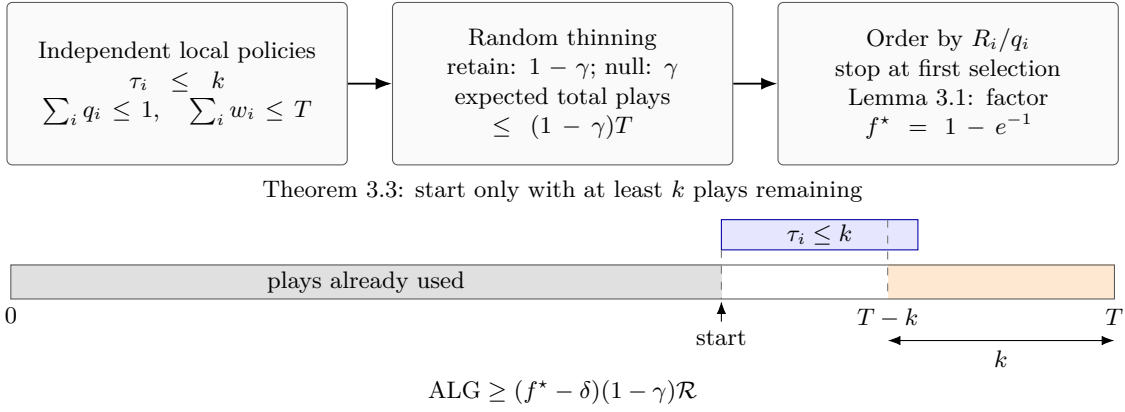
\begin{figure}[htbp]
\centering
\begin{tikzpicture}[x=1cm,y=1cm,font=\small,>=Latex,
 stage/.style={draw=black!65,rounded corners=2pt,fill=black!2,
  align=center,text width=4.15cm,minimum height=2.15cm,inner sep=5pt}]
 \node[stage] (policies) at (2.2,0)
  {Independent local policies\\$\tau_i\le k$\\
   $\sum_iq_i\le1,\quad\sum_iw_i\le T$};
 \node[stage] (thin) at (7.3,0)
  {Random thinning\\retain: $1-\gamma$; null: $\gamma$\\
   expected total plays\\$\le(1-\gamma)T$};
 \node[stage] (order) at (12.4,0)
  {Order by $R_i/q_i$\\stop at first selection\\
   Lemma~\ref{lem:selection-order}: factor\\$\rh=1-e^{-1}$};
 \draw[->,line width=.7pt] (policies.east)--(thin.west);
 \draw[->,line width=.7pt] (thin.east)--(order.west);
 \node at (7.3,-1.42)
  {Theorem~\ref{thm:short}: start only with at least $k$ plays remaining};
 \fill[black!12] (0,-2.85) rectangle (9.4,-2.4);
 \fill[orange!18] (11.6,-2.85) rectangle (14.6,-2.4);
 \draw[black!65] (0,-2.85) rectangle (14.6,-2.4);
 \node at (4.7,-2.625) {plays already used};
 \draw[draw=blue!65!black,fill=blue!10]
  (9.4,-2.2) rectangle (12,-1.8);
 \node at (10.7,-2) {$\tau_i\le k$};
 \draw[black!55,dashed] (9.4,-2.2)--(9.4,-2.85);
 \draw[black!55,dashed] (11.6,-1.8)--(11.6,-2.85);
 \node[below] at (0,-2.85) {$0$};
 \node[below] at (11.6,-2.85) {$T-k$};
 \node[below] at (14.6,-2.85) {$T$};
 \draw[->] (9.4,-3.15)--(9.4,-2.88);
 \node[below] at (9.4,-3.15) {start};
 \draw[<->] (11.6,-3.4)--node[below] {$k$} (14.6,-3.4);
 \node at (7.3,-4.08)
  {$\ALG\ge(\rh-\delta)(1-\gamma)\R$};
\end{tikzpicture}
\caption{A bounded-duration refinement of the stochastic-packing execution
framework of \citet{GM07}; see also \citet[Sections~3.2--3.3]{GM08}.
Thinning scales rewards and selection probabilities equally, preserving
the order $R_i/q_i$. Lemma~\ref{lem:selection-order} bounds the loss
from stopping at the first selection, ignoring the deadline.
Theorem~\ref{thm:short} adds capacity
slack and a $k$-play reserve: every started policy finishes, while the
probability of exclusion by the budget is at most $\delta$.}
\label{fig:short-execution}
\end{figure}

\begin{lemma}[Selection-only order]\label{lem:selection-order}
Order by nonincreasing $R_i/q_i$, omitting $q_i=0$.
With $s_i=\prod_{j<i}(1-q_j)$ and $Q=\sum_iq_i$, we have
 $\sum_iR_i s_i\ge\ph(Q)\R\ge\rh\R.$
\end{lemma}
\begin{proof}
Let $P_k=1-\prod_{i\le k}(1-q_i)$ and
$Q_k=\sum_{i\le k}q_i$. Then $P_k-P_{k-1}=q_ks_k$ and
$P_k\ge1-e^{-Q_k}\ge\ph(Q)Q_k$. Summation by parts with the
nonincreasing densities $R_i/q_i$ proves the result, exactly
as in \eqref{eq:survival}.
\end{proof}

\begin{corollary}[One-play revelation]\label{cor:one-play}
If one play resolves each arm's selectable payoff and $n\le T$,
selection-only ordering obtains at least $\rh\LP$.
The gap $1/\rh$ against $\LP$ is tight.
\end{corollary}
\begin{proof}
An optimal local policy can be reduced to at most one play,
because further observations do not change the selectable value;
its remaining selection decision can be replaced by a coin.
Every policy can finish within the global budget if necessary,
so its admission probability is exactly $s_i$.
Apply Lemma~\ref{lem:selection-order}.

For tightness, recall the needle instance of \citet[Section~3.5]{GM08}:
$n$ independent arms with
$\Prb(\theta_i=1)=1/n$. With $T=n$, $\LP=1$ and
$\OPT=1-(1-1/n)^n\to\rh$. 
With $T=n^2$ the same instance has $u=\log n/n^2\to0$, so the gap persists in
the non-vanishing regime (Corollary~\ref{cor:asymptotic}).
\end{proof}

When $n\le T$, inspecting every arm is optimal, so the content of
Corollary~\ref{cor:one-play} is the integrality gap. When $n>T$,
one-play revelation is closely related to the adaptive ProbeMax problem,
which keeps the best of a bounded number of probed values;
approximation schemes against the optimum are known for ProbeMax
\citep{fu2018ptas,segev2025eptas}. Against $\LP$,
Theorem~\ref{thm:short} with $k=1$ gives
$(\rh-O(\sqrt{\log T/T}))\LP$; see the remark after it.

\begin{theorem}[Short policies]\label{thm:short}
Suppose $\tau_i\le k$ almost surely, where $1\le k\le T$, and let
 $0<\delta<\rh$. Define
\begin{equation}\label{eq:slack}
 \gamma=\frac{k+\sqrt{2kT\log(1/\delta)}+(2k/3)\log(1/\delta)}{T}.
\end{equation}
 If $\gamma<1$, there is a feasible contiguous execution with
 \begin{equation}\label{eq:short-bound}
\ALG\ge(\rh-\delta)(1-\gamma)\R.
\end{equation}
This theorem does not require \eqref{eq:super}.
\end{theorem}

\begin{proof}
Omit policies with $q_i=0$, whose selected reward is zero, and fix
the order by nonincreasing $R_i/q_i$. Independently replace each
policy by the null policy with probability $\gamma$. If $B_i$ is
its retention indicator, then the modified local variables are
\begin{equation}\label{eq:thinning}
 (\bar\tau_i,\bar I_i,\bar Y_i)=B_i(\tau_i,I_i,Y_i),
 \qquad \Prb(B_i=1)=1-\gamma.
\end{equation}
The $B_i$ are independent of all local executions and of each other.
Write $\bar R_i=\E\bar Y_i$ and $\bar q_i=\E\bar I_i$.
Expected total play count is now at most $(1-\gamma)T$, and rewards
and selection probabilities are scaled by $1-\gamma$, preserving
the fixed order.
Run each admitted policy to completion and stop at its selection.
Start a new policy only when the number of plays already used is at
most $T-k$. Thus every admitted policy finishes within $T$ plays.
If no policy selects, choose an arm maximizing its currently
selectable payoff $r_{u_i}$, which is nonnegative.

Let $A_i$ be the event that policy $i$ starts and
$\alpha_i=\Prb(A_i)$. In terms of the complete local variables,
$A_i$ occurs exactly when every earlier $\bar I_j$ is zero and
$\sum_{j<i}\bar\tau_j\le T-k$. Therefore
\[
 \alpha_i\ge\prod_{j<i}(1-\bar q_j)
            -\Prb\left(\sum_{j<i}\bar\tau_j>T-k\right).
\]
Since $0\le\bar\tau_j\le k$, the variance of the sum is at most
$k\sum_j\E\bar\tau_j\le kT$. 
Put $h=\log(1/\delta)$ and $a=\sqrt{2kTh}+2kh/3$. The sum has mean at most
$(1-\gamma)T$, and \eqref{eq:slack} gives $\gamma T-k=a$, so exceeding $T-k$
requires a deviation of more than $a$ above the mean. Bernstein's inequality
gives
\[
 \Prb\left(\sum_{j<i}\bar\tau_j>T-k\right)
 \le\exp\left(-\frac{a^2}{2kT+2ka/3}\right)\le\delta.
\]
The event $A_i$ depends only on earlier policies, so it is independent
of $\bar Y_i$. Every admitted policy finishes, and at most one selects.
The terminal rule therefore obtains expected reward at least
$\sum_i\E[\ind_{A_i}\bar Y_i]=\sum_i\alpha_i\bar R_i$.
Apply Lemma~\ref{lem:selection-order} and
$\sum_i\bar R_i=(1-\gamma)\R$.
\end{proof}
Since $\gamma\ge\sqrt{2k\log(1/\delta)/T}$, the scheduling loss $\gamma$
vanishes if and only if $k\log(1/\delta)=o(T)$.
For $k=1$ and $\delta=1/T$, this yields
$(\rh-O(\sqrt{\log T/T}))\LP$ for one-play revelations,
regardless of whether $n\le T$.

\subsection{Several plays or selections}
\label{app:KL}
Several plays per step model parallel evaluations on $M$ machines. Choosing up to $L$ attributes for a classifier at test time \citep{kapoor2005bounded} (assuming independent features for the present analysis), or retaining up to $L$ independent factors in a response-surface model, motivates allowing multiple selections.
At most $M$ distinct arms may be played per step and at most $L$ arms
selected at the end, with the sum of their selectable payoffs as the
objective. If exactly $L$ selections are required, additional arms
can be selected without reducing the value because payoffs are nonnegative.

\paragraph{Interrupted policies.}
This paragraph concerns the execution of Section~\ref{sec:231}, which
may interrupt a policy at the deadline.
At a stopping history, let $V_j$ be the conditional expected remaining
selected payoff of an unfinished policy $j$ and $r_j$ its current payoff. Under
\eqref{eq:super}, $V_j\le r_j$ at every observed history. With $U\le M$
unfinished policies, selecting all of them covers $\sum_{j=1}^UV_j$.
Selecting only $L<U$ of them guarantees merely
\[
 \sum_{\text{largest }L}r_j\ \ge\ \sum_{\text{largest }L}V_j
 \ \ge\ \frac LU\sum_{j=1}^UV_j,
\]
and the factor $L/U$ is attained when all $V_j$ and $r_j$ equal a
common value.
Completed selections can also use up selection capacity. Reducing the
analysis to $L=1$ while several policies are unfinished therefore loses
exactly the property behind \eqref{eq:admitted}: at most one policy
is unfinished or has selected when that execution stops.

\paragraph{Short (and deferred) policies.}
Short policies avoid interruption, so selection can be deferred to the
end. Let $\mathrm{LP}_{M,L}$ be \eqref{eq:two} with $\sum_iq_i\le L$ and
$\sum_iw_i\le MT$. Each arm is played at most once per step, so
marginalizing a feasible global policy again gives feasible local flows,
and $\OPT\le\mathrm{LP}_{M,L}$. Write $\mathrm{Po}_x$ for a Poisson variable
of mean $x$, $\psi_L(x)=\E\min\{L,\mathrm{Po}_x\}$, and
$c_L=\psi_L(L)/L=1-L^L/(L!\,e^L)$, so that $c_1=\rh$.

\begin{proposition}\label{prop:KL}
Let independent local policies with nonnegative selected rewards satisfy
$\tau_i\le k$, $\sum_iq_i\le L$ and $\sum_iw_i\le MT$, and put
$\R=\sum_iR_i$. For $0<\delta<c_L$ let
\[
 \gamma_M=\frac{Mk+\sqrt{2kMT\log(1/\delta)}+(2k/3)\log(1/\delta)}{MT},
\]
which is \eqref{eq:slack} when $M=1$. If $\gamma_M<1$, there is a
feasible execution with $\ALG\ge(c_L-\delta)(1-\gamma_M)\R$.
\end{proposition}
\begin{proof}
Fix any order of the policies and replace each independently by the
null policy with probability $\gamma_M$. Define
$(\bar\tau_i,\bar I_i,\bar Y_i)$ as in \eqref{eq:thinning}, now with
retention probability $1-\gamma_M$.
At each step, every started and unfinished policy makes one play of
its own arm. While fewer than $M$ policies are running, start the
next listed policy, but only at steps $t\le T-k+1$.
A policy that selects or quits without a play frees its slot immediately.
Every started policy therefore finishes by step $T$, and at most
$M$ arms are played per step. Selection is deferred: a completed
local selection nominates its arm for the final decision.
After step $T$, select the up to $L$ nominated arms with the largest
payoffs, selecting all nominees if there are fewer than $L$ and
none if there are no nominees.

\paragraph{Admission probability.} 
If policy $i$ has not started by an eligible step $t\le T-k+1$,
then at every step up to $t$, $M$ policies listed before $i$ were running.
Hence $i$ starts by step $T-k+1$ whenever
$\sum_{j<i}\bar\tau_j\le M(T-k)$. The event $A_i$ that $i$ starts
depends only on the outcomes and decisions of earlier policies.
The sum has mean at most $(1-\gamma_M)MT$ and variance at most
$kMT$, and $\gamma_M MT-Mk=\sqrt{2kMTh}+2kh/3$ with $h=\log(1/\delta)$.
Bernstein's inequality, as in the proof of Theorem~\ref{thm:short},
bounds the probability that the event fails by $\delta$.

\paragraph{Poissonization.} 
The variables $\bar Y_i$ describe the complete local executions,
whether or not the schedule starts those policies. They are independent
and nonnegative, with $\sum_i\Prb(\bar Y_i>0)\le L$.
Pad the list with zeros when necessary. The sum of its $L$ largest
values is
$\int_0^\infty\min\{L,C_v\}\,dv$, where $C_v=\#\{i:\bar Y_i>v\}$ is a sum
of independent indicators with mean $\mu_v\le L$.
\begin{itemize}
\item Splitting one indicator into two independent indicators with the
same total mean can only lower $\E\min\{L,C_v\}$: only the event that
the other indicators sum to $L-1$ is affected, and there the
probability of at least one success drops from $p+p'$ to
$p+p'-pp'$.
\item Splitting repeatedly gives
$\E\min\{L,C_v\}\ge\psi_L(\mu_v)$
\citep[proof of Lemma~3.2]{chakraborty2010}.
\item $\psi_L$ is concave with $\psi_L(0)=0$, since
$\psi_L'(x)=\Prb(\mathrm{Po}_x\le L-1)$ is nonincreasing. Hence
$\psi_L(\mu_v)\ge c_L\mu_v$ for $\mu_v\le L$.
\end{itemize}
Integrating over $v$, the $L$ largest values of $\bar Y_i$ have expected sum at least
$c_L\sum_i\E\bar Y_i=c_L(1-\gamma_M)\R$. For $L=1$ the first two steps
are not needed: $\Prb(C_v\ge1)\ge1-e^{-\mu_v}$ directly, as in
Lemma~\ref{lem:selection-order}.

\paragraph{Unstarted policies.}
The schedule obtains the sum of the $L$ largest values of
$\ind_{A_i}\bar Y_i$, padded with zeros. Removing the values for
unstarted policies lowers this sum by at most
$\sum_i(1-\ind_{A_i})\bar Y_i$.
Since $A_i$ is independent of $\bar Y_i$ and
$\Prb(A_i^c)\le\delta$, the expected loss is at most
\[
 \sum_i\Prb(A_i^c)\E\bar Y_i
 \le\delta\sum_i\E\bar Y_i=\delta(1-\gamma_M)\R.
\]
Combining the two bounds proves the proposition.
\end{proof}

For $M=L=1$ this recovers Theorem~\ref{thm:short} with a different
execution: no density order, but no stop at the first selection either.
Section~\ref{sec:schedulable} gives the corresponding truncation for
the Bayesian model.

\section{Selection-preserving truncation}
\label{sec:truncation}

We return to $M=L=1$: one play per step and one terminal selection.
For Bayesian arms with exchangeable observations in $[0,1]$,
this section constructs the short policies that
Section~\ref{sec:short} can schedule. We retain each experiment's
prefix and replace its continuation by a coin matching its eventual
selection probability. This preserves the LP's selection constraint
exactly and reduces its play count. Concentration bounds the selected
payoff lost by the replacement; Theorem~\ref{thm:short} then gives the
feasible execution. Figure~\ref{fig:selection-cutoff} shows the local
transformation.

\subsection{Replace continuation by a selection coin}

For the analysis of one arm, define $X_1,\ldots,X_T$ to be iid
observations conditional on its latent distribution, even if the
policy does not take all of them. Let $\F_t$ contain the first $t$
observations and the policy's local randomization through that depth.
The actual policy sees $X_t$ only if it makes its $t$th play.
For a local policy terminating at $\tau$, the tower property gives
$ \E[r_{u_\tau}\ind_{\rm sel}]=\E[\theta_i\ind_{\rm sel}]$.
Follow a policy $P$ until termination or $k$ observations.
If still active at depth $k$, replace all remaining actions
by an independent selection coin of probability

\begin{equation}\label{eq:coin}
 \chi_k=\Prb(P\text{ eventually selects}\mid\F_k).
\end{equation}
Select the current arm on a positive coin and otherwise quit;
call this policy $P^{[k]}$. The coin is independent of future observations
and future randomization of $P$, conditional on $\F_k$.

\begin{figure}[htbp]
\centering
\begin{tikzpicture}[x=1cm,y=1cm,font=\small,>=Latex,
 state/.style={circle,draw=black!60,fill=white,minimum size=2.5mm,inner sep=0pt},
 terminal/.style={draw=black!55,rounded corners=1pt,fill=white,
  inner sep=3pt,font=\scriptsize}]
 \fill[black!4] (1.85,-3.08) rectangle (6.35,-1.36);
 \node at (3.15,.52) {Original policy $P$};
 \node at (10.65,.52) {Truncated policy $P^{[k]}$};
 \draw[black!45,dashed] (.1,-1.1)--(6.5,-1.1);
 \draw[black!45,dashed] (7.9,-1.1)--(14,-1.1);
 \node[left] at (.1,-1.1) {depth $k$};
 \node[state] (lr) at (2.8,0) {};
 \node[terminal] (le) at (1.1,-.57) {early select/quit};
 \node[state,minimum size=5mm] (lu) at (4.1,-1.1) {$u$};
 \draw[->] (lr)--(le);
 \draw[->] (lr)--node[midway,sloped,above,font=\scriptsize] {prefix} (lu);
 \node[state] (lv0) at (3,-1.97) {};
 \node[state] (lv1) at (5.2,-1.97) {};
 \node[terminal] (ls0) at (2.3,-2.78) {select};
 \node[terminal] (lq0) at (3.45,-2.78) {quit};
 \node[terminal] (ls1) at (4.65,-2.78) {select};
 \node[terminal] (lq1) at (5.8,-2.78) {quit};
 \draw[->] (lu)--(lv0);
 \draw[->] (lu)--(lv1);
 \draw[->] (lv0)--(ls0);
 \draw[->] (lv0)--(lq0);
 \draw[->] (lv1)--(ls1);
 \draw[->] (lv1)--(lq1);
 \node at (4.1,-3.4) {further observations};
 \node[state] (rr) at (10.3,0) {};
 \node[terminal] (re) at (8.6,-.57) {early select/quit};
 \node[state,minimum size=5mm] (ru) at (11.6,-1.1) {$u$};
 \draw[->] (rr)--(re);
 \draw[->] (rr)--node[midway,sloped,above,font=\scriptsize] {prefix} (ru);
 \node[terminal] (rs) at (10.4,-2.55) {select at $u$};
 \node[terminal] (rq) at (12.8,-2.55) {quit};
 \draw[->] (ru)--node[pos=.52,left] {$\chi_k$} (rs);
 \draw[->] (ru)--node[pos=.52,right] {$1-\chi_k$} (rq);
 \node at (11.6,-3.4) {no further observations};
 \draw[->,line width=.8pt] (6.65,-1.72)--
  node[above,align=center,font=\scriptsize] {replace\\continuation} (8,-1.72);
 \node at (6.7,-4.05)
  {$I(P^{[k]})=I(P),\qquad W(P^{[k]})\le W(P),\qquad
    \tau(P^{[k]})\le k$};
\end{tikzpicture}
\caption{Selection-preserving truncation at an active cutoff state
$u$. The coin uses the conditional selection probability $\chi_k$
from \eqref{eq:coin}; all earlier decisions remain the same.
Selection probability is preserved exactly and play count does not increase,
while the following lemmas control selected-payoff loss.}
\label{fig:selection-cutoff}
\end{figure}
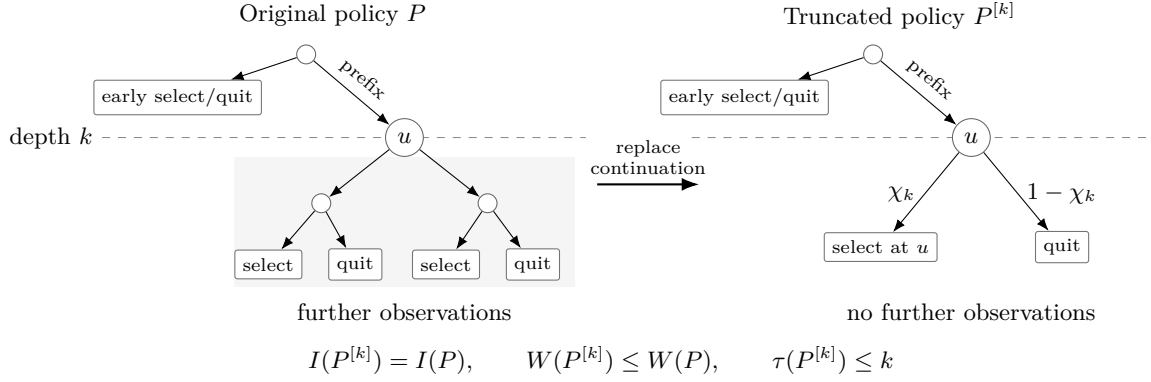

Let $C$ be the event that $P$ is still active when the cutoff is reached,
let $S=C\cap\{P\text{ eventually selects}\}$ be the part of the selection
event that the cutoff affects, and set $Z=\chi_k\ind_C=\E[\ind_S\mid\F_k]$.
Before the cutoff, $P$ and $P^{[k]}$ make the same selections.

Let $\widehat\theta_k$ be the empirical mean of the first $k$ observations;
it is $\F_k$-measurable. Then
\begin{equation}\label{eq:match}
 I(P^{[k]})=I(P),\qquad W(P^{[k]})\le W(P),\qquad
 \tau(P^{[k]})\le k,\qquad
 \E[\widehat\theta_k Z]=\E[\widehat\theta_k\ind_S].
\end{equation}
For $k\ge1$, the event of making a first play is also unchanged.
The coin is used only in the analysis; the algorithm solves the
depth-$k$ LP directly.

\subsection{Charge reward loss to selection mass}

\begin{lemma}\label{lem:estimate}
For $0<\epsilon<1/4$, $a_0>0$, and any fixed latent
distribution of mean $\theta$,
\[
 \Prb\bigl(|\widehat\theta_k-\theta|>
       \epsilon\max\{\theta,a_0\}\bigr)
 \le\zeta:=2e^{-k\epsilon^2a_0/3}.
\]
\end{lemma}
\begin{proof}
Conditional variance of an observation is at most $\theta$.
Apply the two-sided Bernstein bound at
$t=\epsilon\max\{\theta,a_0\}$. If $\theta\ge a_0$, its exponent
$kt^2/(2\theta+2t/3)$ is at least $k\epsilon^2\theta/3$.
If $\theta<a_0$, it is at least $k\epsilon^2a_0/3$.
Both statements use $\epsilon<1/4$, which gives
$2\theta+2\epsilon\max\{\theta,a_0\}/3\le3\max\{\theta,a_0\}$.
\end{proof}

 \begin{lemma}[Reward preservation]\label{lem:trunc}
Under the conditions of Lemma~\ref{lem:estimate}, every local policy $P$ and
 its truncation $P^{[k]}$ satisfy
\begin{equation}\label{eq:trunc-bound}
 \widetilde R(P^{[k]})\ge
 \frac{(1-\epsilon)\widetilde R(P)
       -2\epsilon a_0 I(P)-2\zeta}{1+\epsilon}.
\end{equation}
\end{lemma}
\begin{proof}
On the concentration event,
$(1-\epsilon)\theta-\epsilon a_0\le\widehat\theta_k
\le(1+\epsilon)\theta+\epsilon a_0$.
All means and weights lie in $[0,1]$, so these pointwise
bounds give the unconditional inequalities
\begin{align*}
 \E[\widehat\theta_k Z]
   &\le(1+\epsilon)\E[\theta Z]
        +\epsilon a_0\Prb(S)+\zeta,\\
 \E[\widehat\theta_k\ind_S]
   &\ge(1-\epsilon)\E[\theta\ind_S]
        -\epsilon a_0\Prb(S)-\zeta.
\end{align*}
Equate the left sides using \eqref{eq:match}, then use
tower property on the latent mean. The policies agree before the cutoff,
and $\Prb(S)\le I(P)$, giving \eqref{eq:trunc-bound}.
\end{proof}
\begin{corollary}[Two LP error bounds]\label{lem:depth-k}
\label{prop:additive-trunc}
Let $\LP^{(k)}$ restrict \eqref{eq:two} to depth-$k$ policies. If $\LP>0$ and
 $0<\epsilon<1/4$, then
 \begin{equation}\label{eq:relative-trunc}
  \LP^{(k)}\ge(1-3\epsilon)\LP-4n e^{-k\epsilon^2\LP/6}.
 \end{equation}
 For every $z>0$, including when $\LP=0$,
  \begin{equation}\label{eq:additive-trunc}
     0\le\LP-\LP^{(k)}\le2z+2n e^{-2kz^2}.
   \end{equation}
      The two LPs coincide when $k\ge T$.
\end{corollary}
\begin{proof}
The truncated optimal policies remain feasible by
\eqref{eq:match}. Sum \eqref{eq:trunc-bound} with
$a_0=\LP/2$ and $\sum_i I(P_i)\le1$ to obtain
\eqref{eq:relative-trunc}.

For one local policy use the notation of \eqref{eq:match}:
$S$ is its selection event after reaching the cutoff,
$Z=\E[\ind_S\mid\F_k]$, and $\widehat\theta_k$ is the empirical
mean of the first $k$ observations. The policies agree before the
cutoff. The tower property applied to the latent mean and
$\E[\widehat\theta_k(\ind_S-Z)]=0$ give
\[
 \widetilde R(P)-\widetilde R(P^{[k]})
 =\E[(\theta-\widehat\theta_k)(\ind_S-Z)].
\]
Conditional on the latent reward distribution, Hoeffding's
inequality gives
$\Prb(|\theta-\widehat\theta_k|>z)\le2e^{-2kz^2}$.
On its complement, the integrand is at most
$z(\ind_S+Z)$. On the exceptional event its absolute value is at
most one, since $\theta,\widehat\theta_k,Z\in[0,1]$.
As $\E Z=\Prb(S)\le I(P)$,
\[
 \widetilde R(P^{[k]})
 \ge\widetilde R(P)-2zI(P)-2e^{-2kz^2}.
\]
Sum over optimal local policies of $\LP$. Their truncated versions
remain feasible by \eqref{eq:match}, and $\sum_i I(P_i)\le1$,
proving \eqref{eq:additive-trunc}. No local policy is shortened
when $k\ge T$.
\end{proof}

The relative bound \eqref{eq:relative-trunc} is used for
approximation below, and the additive bound \eqref{eq:additive-trunc}
for the certificates of Section~\ref{sec:extensions}.

\subsection{Short and schedulable}
\label{sec:schedulable}
Recall $\rh=1-e^{-1}$.
Fix $0<\epsilon<1/4$, $\LP>0$, and set
\begin{equation}\label{eq:parameters}
 k=\left\lceil\frac{6\log(4n/(\epsilon\LP))}
                         {\epsilon^2\LP}\right\rceil,
 \qquad\delta=\epsilon\rh.
\end{equation}
Then Corollary~\ref{lem:depth-k} gives
$\LP^{(k)}\ge(1-4\epsilon)\LP$.
Solve this truncated LP once. Independently thin its local
policies by the slack $\gamma$ in \eqref{eq:slack}, then use
the contiguous execution of Theorem~\ref{thm:short}.
Every admitted policy finishes within the horizon.
Use Section~\ref{sec:231} when $\gamma\ge1$.

\begin{theorem}\label{thm:asymptotic}
  Let $k$ and $\delta$ be as in \eqref{eq:parameters} and $\gamma$ as in
  \eqref{eq:slack}. If $\gamma<1$, the algorithm above satisfies
  \begin{equation}\label{eq:main}
   \ALG\ge\rh(1-5\epsilon-\gamma)\LP.
  \end{equation}
  \end{theorem}
  \begin{proof}
Theorem~\ref{thm:short} and Corollary~\ref{lem:depth-k} give
$\ALG\ge\rh(1-\epsilon)(1-\gamma)(1-4\epsilon)\LP$.
Expanding the nonnegative factors proves the bound.
\end{proof}

\begin{corollary}\label{cor:asymptotic}
Let $n\ge2$ and $u=\log(n/\LP)/(T\LP)$.
As $u\to0$, choosing $\epsilon=u^{1/4}$ gives
\[
 \ALG\ge\bigl(\rh-O(u^{1/4}\log(1/u))\bigr)\LP.
\]
The limiting LP gap $1/\rh$ is tight within this regime.
\end{corollary}

\begin{proof}
For small $u$, $\epsilon<1/4$. Since $\LP\le1$ and $n\ge2$,
$1/T\le1/(T\LP)\le u/\log2$, so with the ceiling
$k/T\le6(\log(n/\LP)+\log(4/\epsilon))/(\epsilon^2T\LP)+1/T
=O(u^{1/2}\log(1/u))$. As $\log(1/\delta)=O(\log(1/u))$,
\eqref{eq:slack} gives $\gamma=O(u^{1/4}\log(1/u))<1$; apply
Theorem~\ref{thm:asymptotic}. Corollary~\ref{cor:one-play}'s needles
with $T=n^2$ have $\LP=1$ and $u=\log n/n^2\to0$.
\end{proof}

Planning uses only depth-$k$ states, with no explicit horizon
dependence beyond the LP value. For Beta--Bernoulli arms this
is $O(nk^2)$ states; more generally, a compact state space or
local optimization oracle remains necessary.

\paragraph{Several plays or selections.}
For the model of Section~\ref{app:KL}, sum \eqref{eq:trunc-bound} with
$a_0=\mathrm{LP}_{M,L}/(2L)$ and $\sum_iI(P_i)\le L$. As in
Corollary~\ref{lem:depth-k}, depth
$k=\lceil6L\log(4n/(\epsilon\mathrm{LP}_{M,L}))/(\epsilon^2\mathrm{LP}_{M,L})\rceil$,
which does not depend on $T$, keeps $(1-4\epsilon)\mathrm{LP}_{M,L}$. With
$\delta=\epsilon c_L$, Proposition~\ref{prop:KL} then gives
\[
 \ALG\ \ge\ c_L(1-5\epsilon-\gamma_M)\mathrm{LP}_{M,L}.
\]
Here $\gamma_M=k/T+O\big(\sqrt{k\log(1/\delta)/(MT)}\big)$: several plays
per step divide the work, but not the $k$ rounds of adaptivity within a
local policy.

\subsection{Choosing the cutoff without knowing the full LP}
The full optimal policies are used only to prove that the
depth-$k$ LP has nearly their value. They need not be constructed.
If $\LP$ is unknown, double $k$ until
\[
 k\ge\frac{6\log(4n/(\epsilon\LP^{(k)}))}
                   {\epsilon^2\LP^{(k)}}.
\]
The right side is decreasing in $\LP^{(k)}>0$,
and $\LP^{(k)}\le\LP$. So this test implies
\eqref{eq:parameters}'s sufficient cutoff condition.
Stop also when $k=T$, where the full LP is reached and no
truncation error remains, including when $\LP=0$.
Let $k_0$ be the cutoff in \eqref{eq:parameters} and
$\epsilon\le1/8$. Since $\LP^{(k)}$ is nondecreasing in $k$, every
$k\ge k_0$ has $\LP^{(k)}\ge(1-4\epsilon)\LP$, and then the right
side above is at most $4k_0$. The test therefore succeeds for every
$k\ge4k_0$, and doubling stops with $k<8k_0$ after $O(\log k_0)$
LP solutions.

Alternatively, any positive lower bound $\underline{\LP}\le\LP$ can
replace $\LP$ in the cutoff formula. Use $a_0=\underline{\LP}/2$ in
\eqref{eq:trunc-bound}.
The term $2\epsilon a_0I(P)$, summed over arms, is at most
$\epsilon\underline{\LP}\le\epsilon\LP$ and the exponential term is
at most $\epsilon\underline{\LP}$. Thus the same $(1-4\epsilon)\LP$
retention bound follows. The largest prior mean supplies
such a lower bound whenever it is positive.

\section{Certificates for sequential evaluation}
\label{sec:extensions}

The preceding sections use the LP to design policies and control
their execution loss. The LP answers a decision
question: how much can further evaluation improve the current
choice? This certificate can accompany any sampling rule. Local
truncation makes it computable without exploring the full remaining
horizon, while the dual bounds every feasible adaptive continuation.

Throughout this section, observations and latent means lie in $[0,1]$,
as in Section~\ref{sec:truncation}, and arms are independent. In
particular, evaluations that share data across arms, such as scoring
several models on the same test examples, are outside this model.
At history $H$, let $B$ be the remaining budget,
$b(H)=\max_i\E[\theta_i\mid H]\in[0,1]$ the value of stopping, and
$M(H)=\E[\max_i\theta_i\mid H]$ the perfect-information value.
The expected simple regret of stopping decomposes as
\[
 M(H)-b(H)=\bigl[M(H)-\OPT_B(H)\bigr] +\bigl[\OPT_B(H)-b(H)\bigr].
\]
The first term is the least regret achievable under the remaining
budget; the second is the improvement still available through
evaluation. A large posterior regret can therefore coexist with
little attainable improvement. The perfect-information gap $M(H)-b(H)$,
used among other bounds by \citet{hay2012computations}, bounds the
second term but ignores the remaining budget; the LP bounds it directly.

For fixed $\eta,\lambda\ge0$ let
$D_{B,k}(H)=\eta+\lambda B+\sum_iG^{(k)}_{i,H}(\eta,\lambda)$: the objective
 of \eqref{eq:two-dual} with $T$ replaced by $B$, each arm re-rooted at its
 state under $H$, and at most $\min\{k,B\}$ further plays per arm.

For $0<\epsilon<1$, the explicit additive depth is
\begin{equation}\label{eq:additive-depth}
 k_\epsilon=
 \left\lceil\frac{8}{\epsilon^2}
                    \log\frac{4n}{\epsilon}\right\rceil.
\end{equation}
Setting $z=\epsilon/4$ in \eqref{eq:additive-trunc} makes each error
term at most $\epsilon/2$ at this depth, so
$\LP\le\LP^{(k_\epsilon)}+\epsilon$. When $k\ge B$, no continuation is
shortened.

\begin{proposition}[Certificate for stopping]\label{cor:regret-stop}
Consider independent Bayesian arms with exchangeable observations in
 $[0,1]$. Let $0<\epsilon<1$ and $k\ge k_\epsilon$, with $k_\epsilon$
as in \eqref{eq:additive-depth}. At every
history $H$,
\begin{equation}\label{eq:continuation-certificate}
\OPT_B(H)\le U_{B,k}(H):=\min\{1,D_{B,k}(H)+\epsilon\}.
 \end{equation}
 Let $\pi$ be any feasible policy and $r>0$. Let $\pi_{\rm stop}$ follow $\pi$
until the first history $H$ with $
 U_{B,k}(H)-b(H)\le r$,
and then stop and select an arm of largest posterior mean.
If no test succeeds before $\pi$ terminates, make the same terminal
selection as $\pi$.
Let $\operatorname{SR}(\pi)=\E[\max_i\theta_i-\theta_{J_\pi}]$ and
$J_\pi$ be the arm selected by $\pi$, then
 \begin{equation}\label{eq:regret-stop-bound}
\operatorname{SR}(\pi_{\rm stop})\le\operatorname{SR}(\pi)+r,
 \end{equation}
\end{proposition}

\begin{proof}
The bound applies to every collection of independent posteriors,
with the same $k_\epsilon$. Conditioning on an adaptively
collected history preserves the product posterior in this model:
the action probabilities do not depend on the latent distributions
once the observed history is fixed, and the observation likelihood
factors by arm. Thus one can re-root the model at any history and
apply Corollary~\ref{prop:additive-trunc}. For $k=k_\epsilon$ from
\eqref{eq:additive-depth}, the resulting depth-$k$ dual gives
\[
 D_{B,k}(H)=\eta+\lambda B+
             \sum_iG^{(k)}_{i,H}(\eta,\lambda),
 \qquad \eta,\lambda\ge0,
\]
where the local recursions allow at most $\min\{k,B\}$ further
plays. Weak duality and Corollary~\ref{prop:additive-trunc} imply
\eqref{eq:continuation-certificate}. Any pair of nonnegative prices
is valid; optimizing the dual tightens the bound. The error is zero
when $B\le k$, or whenever the truncated and full LPs coincide.
For Bernoulli observations with a
predictive-probability oracle there are $O(nk^2)$ additional count
states.

At a fixed history $H$, stopping has conditional expected simple
regret $M(H)-b(H)$; the best continuation has regret
$M(H)-\OPT_B(H)$. Their difference is
$\OPT_B(H)-b(H)\le U_{B,k}(H)-b(H)\le r$.

For the policy statement, let $\tau$ be the first time $t<T$ at
which the test succeeds while $\pi$ is still running, and set
$\tau=T$ if there is no such time. Couple $\pi$ and $\pi_{\rm stop}$
until this test succeeds or $\pi$ terminates. On $\{\tau<T\}$,
conditional on $H_\tau$, the continuation of $\pi$ is feasible for budget
$T-\tau$, so
$\E[\theta_{J_\pi}\mid H_\tau]\le\OPT_{T-\tau}(H_\tau)\le b(H_\tau)+r$,
whatever selection rule $\pi$ uses. The stopped policy has
$\E[\theta_{J_{\rm stop}}\mid H_\tau]=b(H_\tau)$. If no early stop
occurs, the two policies select the same arm. Taking expectations gives
$\E[\theta_{J_\pi}]-\E[\theta_{J_{\rm stop}}]\le r\Prb(\tau<T)\le r$,
and subtracting both from $\E\max_i\theta_i$ gives
\eqref{eq:regret-stop-bound}. The certificate holds at every
history, so repeated checking introduces no additional error.
\end{proof}

\paragraph{What the certificate establishes.}
For any continuation with certified conditional reward lower bound
$\underline V(H)$, its deficit from the best budget-feasible continuation is
at most $U_{B,k}(H)-\underline V(H)$. The same certificate proves
$\inf_\pi\operatorname{SR}(\pi\mid H)\ge[M(H)-U_{B,k}(H)]^+$:
it can establish that the available budget cannot achieve a
specified regret target. Thus evaluation can report both an
achievable value and a ceiling on further improvement. A loose
relaxation can only delay stopping.

\subsection{Switch-aware continuation certificates}

Fix an observed history $H$, an initially configured arm $j$,
and nonnegative integer budgets $B$ and $S$ for future plays
and changes of evaluation arm. Final recommendation itself
does not require another evaluation. The joint relaxation is
\[
 \LP_{B,S}(H)=\max\left\{\sum_iR_i:
   \sum_iq_i\le1,\quad\sum_iw_i\le B,
   \quad\sum_{i\ne j}a_i\le S\right\},
 \qquad a_i=\Prb(\tau_i\ge1).
\]
Use local depth at most $B$ in the full LP and
$\min\{k,B\}$ in $\LP^{(k)}_{B,S}(H)$.
Every distinct noncurrent arm played by a global continuation
requires at least one switch, so the pathwise count of such
arms is at most $S$. Marginalization gives the expected
activation constraint above, in addition to the usual play
and selection constraints. Thus
$\OPT_{B,S}(H)\le\LP_{B,S}(H)$.
Repeated returns consume switches that this relaxation ignores.

In the dual, let $G_i(u;\eta,\lambda)$ be the ordinary
depth-restricted recursion after the first future play.
At the re-rooted state $u_i$, use
\[
 G^{(k)}_{i,H}(\eta,\lambda,\kappa)=\max\left\{
   0,\ r_{u_i}-\eta,
   -\lambda-\kappa\ind_{i\ne j}
       +\sum_vp_{u_iv}G_i(v;\eta,\lambda)\right\}.
\]
The play option is omitted when no plays remain.
The price $\kappa$ is charged only on the first future play.
Weak duality gives
\begin{equation}\label{eq:switch-dual}
 D_{B,S,k}(H)=\eta+\lambda B+\kappa S+\sum_iG^{(k)}_{i,H}(\eta,\lambda,\kappa).
\end{equation}
At $k\ge1$, truncation agrees with the original policy
until its first play or its earlier terminal decision, so
it preserves $a_i$ as well as $q_i$, and does not increase $w_i$.
The proof of \eqref{eq:additive-trunc} therefore applies
unchanged to this extra-constrained LP. Consequently
\[
 \OPT_{B,S}(H)\le
 \min\{1,D_{B,S,k}(H)+\epsilon\}
 \quad\text{for }k=k_\epsilon.
\]
The product posterior is preserved on conditioning on $H$,
as above. Coupling any switch-feasible policy to its early-stopped
version now proves \eqref{eq:regret-stop-bound} under both
budgets.

\paragraph{A resource-sensitive illustration.}
A known, currently configured arm has mean $b$. Each of the other
$n-1$ arms is uninspected and independently has mean one with
probability $p\le b$ and zero otherwise, revealed by one play. Let
$\bar n=\min\{B,S,n-1\}$ and assume $\bar np\le1$. Every arm is
resolved by at most one play, so depth one incurs no truncation loss,
and the certificate holds with $\epsilon=0$:
\[
 \OPT_{B,S}\le\LP^{(1)}_{B,S}=b+(1-b)\bar np.
\]

For the dual, $a_i=w_i$ for uncertain arms, so they share a total
inspection-mass constraint $\sum_i a_i\le\bar n$. Use $\eta=b$ and
assign price $(1-b)p$ to whichever of $B$ and $S$ equals $\bar n$,
with zero price on the other. Every local surplus is zero: playing
an uncertain arm yields expected surplus $(1-b)p$ before
this price, and selecting it unobserved yields $p-b\le0$.
If instead $\bar n=n-1<\min\{B,S\}$, take $\eta=b$ and
zero resource prices; each uncertain local surplus is
$(1-b)p$, yielding the same bound.

To attain this LP value, choose $\bar n$ uncertain arms. Each plays
once and selects exactly when its revealed mean is one; all other
uncertain arms quit at the root. The known arm selects at its root
with probability $1-\bar np$ and otherwise quits. These local policies
have total selection probability $\bar np+(1-\bar np)=1$, expected
work and noncurrent activation mass both equal to $\bar n$, and reward
$\bar np+b(1-\bar np)=b+(1-b)\bar np$. They are feasible because
$\bar n\le\min\{B,S,n-1\}$ and $\bar np\le1$, proving equality for the LP.
The actual optimum is
$b+(1-b)[1-(1-p)^{\bar n}]$, because each inspected arm has
the same independent chance of success and further inspection
helps only before the first success.

For $b=0.6$, $p=0.001$, $n-1=10{,}000$, $B=1{,}000$, and $S=10$,
the certified gain is at most $0.004$, so $r=0.005$ permits
stopping. Yet $M-b=(1-b)[1-(1-p)^{n-1}]\approx0.4$. With ten
switches, most of this uncertainty cannot be resolved, however large
$B$ is.

\subsection{Scheduling short policies with setup costs}
Recall $\rh=1-e^{-1}$.
Let $c_i\ge0$ be the cost of configuring arm $i$ upon
entry from another arm. The initially configured arm $j$
can be evaluated without a setup cost. Observations cost one,
and a hard combined budget $T$ includes observations and setups.
Setup costs consume this budget; the terminal payoff remains $r_u$.
For local policies write $a_i=\Prb(\tau_i\ge1)$, $c_j=0$
in the relaxation, and impose
\[
 \sum_iq_i\le1,\qquad
 \sum_i(w_i+c_i a_i)\le T.
\]
This is a relaxation of all policies in the setup-cost model:
every other visited arm incurs at least one setup, and repeated
entries only increase actual cost. Root selection or quitting
requires no evaluation setup.

\begin{proposition}\label{prop:setup-short}
For independent local policies of depth at most $k$, let
$c_{\max}=\max_i c_i$, $k_c=k+c_{\max}$, and
$\R_0=\sum_iR_i$. For $0<\delta<\rh$, define
\[
 \gamma_c=\frac{k_c+\sqrt{2k_c T\log(1/\delta)}
                         +(2k_c/3)\log(1/\delta)}{T}.
\]
If $\gamma_c<1$, a feasible schedule with no returns obtains
at least $(\rh-\delta)(1-\gamma_c)\R_0$.
\end{proposition}
\begin{proof}
Fix an order with the current arm first and the others in any order.
Independently replace each policy by the null policy with probability
$\gamma_c$, and define $(\bar\tau_i,\bar I_i,\bar Y_i)$ as in
\eqref{eq:thinning}, with retention probability $1-\gamma_c$.
Run each admitted policy to completion in one contiguous block,
admitting only while cost already used is at most $T-k_c$.
A completed local selection nominates its arm; continue evaluation
without returning to any completed arm. Finally select an arm with
the largest currently selectable payoff $r_{u_i}$, including when
no arm was nominated. This payoff is nonnegative and is at least
the payoff of every nominee.
Each arm is evaluated in at most one block. Its complete local
block cost is $C_i=\tau_i+c_i\ind_{\tau_i\ge1}\le k_c$, and
its cost after thinning is $\bar C_i=B_iC_i$.
Initial execution of arm $j$ has no setup fee because $c_j=0$.

The variables $\bar C_i$ are independent, have total mean at most
$(1-\gamma_c)T$ and total variance at most $k_c T$.
Bernstein's inequality bounds the probability that any given policy
is not admitted by $\delta$, exactly as in Theorem~\ref{thm:short}.
Let $A_i$ be the event that policy $i$ is admitted. It depends only
on earlier policies, so it is independent of $\bar Y_i$.

The variables $\bar Y_i$ are independent nonnegative selected payoffs
from the complete local executions. For every $v\ge0$, put
$\mu_v=\sum_i\Prb(\bar Y_i>v)\le\sum_i\bar q_i\le1$,
where $\bar q_i=\E\bar I_i$. Hence
\[
 \Prb(\max_i\bar Y_i>v)
  \ge1-e^{-\mu_v}\ge\rh\mu_v.
\]
Integrating gives
$\E\max_i\bar Y_i\ge\rh\sum_i\E\bar Y_i
=\rh(1-\gamma_c)\R_0$.
The schedule's selected payoff is at least
$\max_i\ind_{A_i}\bar Y_i$. Removing values for unadmitted policies
lowers the maximum by at most $\sum_i(1-\ind_{A_i})\bar Y_i$,
whose expectation is
\[
 \sum_i\Prb(A_i^c)\E\bar Y_i
 \le\delta\sum_i\E\bar Y_i=\delta(1-\gamma_c)\R_0.
\]
Combining these bounds proves the proposition.
\end{proof}

For $k\ge1$, the selection coin preserves activation and
reduces work, so Corollary~\ref{lem:depth-k} also holds
for the setup-cost LP. Solve its truncated version and
apply Proposition~\ref{prop:setup-short}. In particular,
the relative cutoff of \eqref{eq:parameters} gives
$\ALG\ge\rh(1-5\epsilon-\gamma_c)\LP_{\rm setup}$
with $\delta=\epsilon\rh$. Schedulability now requires
the entire block, including setup, to be short relative
to the resource budget.

\subsection{Metric movement and routing certificates}
Let evaluation arms have a fixed metric $d$, current location $j$,
and a remaining hard movement budget $\Delta$, separate from the
remaining play budget $B$. Let $\mathcal Q_\Delta\subseteq[0,1]^n$
be a compact convex set that contains, for every walk from $j$ of length
at most $\Delta$, the incidence vector of the set of arms the walk visits.
Define its support function by
\[
 \sigma_{\mathcal Q_\Delta}(\boldsymbol\kappa)
   =\max_{z\in\mathcal Q_\Delta}\sum_i\kappa_i z_i.
\]
For an exact visitation relaxation this is the rooted
orienteering objective, as in \citet{GM09metric}; an outer
relaxation can instead give an efficiently computable upper bound.
The movement budget couples arms through routing, while their
latent distributions remain independent.

\begin{proposition}\label{prop:metric-certificate}
In the bounded exchangeable Bayesian model, take nonnegative
prices $\eta$, $\lambda$, and $\boldsymbol\kappa$. Let
$G^{(k)}_{i,H}(\eta,\lambda,\kappa_i)$ be the local surplus
over at most $\min\{k,B\}$ future observations, charging
$\kappa_i$ on the first future play. For $k=k_\epsilon$,
\[
 \OPT_{B,\Delta}(H)\le\min\{1,D_{B,\Delta,k}(H)+\epsilon\},
\]
where
\begin{equation}\label{eq:metric-dual}
 D_{B,\Delta,k}(H)=\eta+\lambda B
    +\sigma_{\mathcal Q_\Delta}(\boldsymbol\kappa)+\sum_iG^{(k)}_{i,H}(\eta,\lambda,\kappa_i).
\end{equation}
Consequently the early-stopping guarantee of
Proposition~\ref{cor:regret-stop} holds under both budgets.
\end{proposition}
\begin{proof}
For each feasible adaptive continuation, let $z_i$ be the
probability that its movement walk visits arm $i$, and
$a_i$ the probability of making at least one future play
of that arm. Then $a_i\le z_i$ and
$z\in\mathcal Q_\Delta$, by convexity and the hard movement
constraint. Marginalizing gives feasible local policies
for the relaxation with selection mass at most one,
expected work at most $B$, and $a\le z\in\mathcal Q_\Delta$.
For any such feasible collection,
\[
 \sum_iR_i\le\eta+\lambda B
     +\sum_i\kappa_i z_i
     +\sum_i(R_i-\eta q_i-\lambda w_i-\kappa_i a_i).
\]
Maximizing each local term and the routing term gives
\eqref{eq:metric-dual} as a valid dual upper bound for
the depth-$k$ relaxation. Truncation preserves $a_i$ and
$q_i$, reduces $w_i$, and can keep the same routing vector
$z$. Hence the additive calculation of
Corollary~\ref{prop:additive-trunc} bounds the full
relaxation by its truncated value plus $\epsilon$.
This proves the continuation bound. Conditioning at the
first successful stopping test, as in
Proposition~\ref{cor:regret-stop}, proves the regret claim.
\end{proof}

Any certified upper bound on
$\sigma_{\mathcal Q_\Delta}(\boldsymbol\kappa)$ may replace
it in \eqref{eq:metric-dual}. When $\mathcal Q_\Delta$ is the exact
visitation convex hull, an $\alpha$-approximate orienteering
route of value $V\ge\sigma_{\mathcal Q_\Delta}(\boldsymbol\kappa)/\alpha$
gives the upper bound $\alpha V$. For an outer relaxation,
one must instead upper-bound that relaxation's support value.

\citet[Lemma~1 and Corollary~1]{GM09metric} obtain the no-return
transformation and the orienteering formulation for their Lagrangian
objective, with hard distance control. What is added here is
truncation that preserves both selection and activation probabilities,
so that the remaining-budget certificate can be evaluated on short
local state spaces.

\section{Beyond posterior means: Terminal utilities and prediction}
\label{sec:terminal-scores}

Sequential experiments can support many terminal decisions besides
choosing the largest posterior mean. Expected utility, threshold
satisfaction, and prediction under a specified loss are established
decision-theoretic criteria. The value of information for such decisions
is classical \citep{blackwell1953,degroot1962}. This section explains
which of the preceding approximation arguments apply to these criteria
and bounds the loss from shortening the local experiments they induce.

Keep one play per step and one terminal selection. At a final history
$H$, arm $i$ has a known nonnegative score $\psi_i(H)$, determined by
its local observations. The objective is
$\E[\max_i\psi_i(H)]$. In the LPs and algorithms, replace $r_u$ by
the terminal score at state $u$ and compare those scores at the final
choice. Observations collected during exploration do not themselves
contribute to the objective.

For one Bayesian arm, let $t$ count its local observations and let
$\F_t$ be its local filtration, as in Section~\ref{sec:truncation}.
Observations lie in $[0,1]$ and are iid given a latent distribution
$D$ with mean $\theta$. Write
$\mu_t=\E[\theta\mid\F_t]$ and $v_t=\Var(\theta\mid\F_t)$.

\subsection{Established terminal decision criteria}

\paragraph{Expected utility and satisficing.}
For a bounded nonnegative utility $g$, use
$\psi_t=\E[g(\theta)\mid\F_t]$; examples include
$g(\theta)=\sqrt\theta$ and $g(\theta)=\theta^2$.
A function $f(\mu_t)$ instead evaluates the reported posterior mean
and generally differs from $\E[f(\theta)\mid\F_t]$.

Satisficing uses an aspiration level rather than insisting on exact
maximization \citep{simon1955}. In bandits, \citet{reverdy2017}
study satisfaction thresholds for both mean and realized rewards,
and \citet{russo2022satisficing} study time-sensitive learning of
satisfactory actions. Related PAC selection and stopping guarantees
seek an approximately optimal arm with prescribed confidence
\citep{evendar2006}. Here we use a fixed-budget Bayesian terminal
criterion. For a quality threshold $\theta_0\in[0,1]$, set
\[
 g(\theta)=\ind\{\theta\ge\theta_0\},\qquad
 \psi_t=\Prb(\theta\ge\theta_0\mid\F_t).
\]
For the arm $J$ selected at the final history $H$, the tower property
gives $\E[\psi_J(H)]=\Prb(\theta_J\ge\theta_0)$. Thus the objective
maximizes the probability that the final choice meets the threshold.
A prescribed posterior-confidence requirement would additionally
need a constraint.

\paragraph{Selection accompanied by a prediction.}
Under squared-error loss, the posterior mean is the classical Bayes
prediction. \citet{savage1971} characterizes losses whose expected
value is minimized by reporting the mean;
\citet[Section~3.1]{gneiting2011} gives a modern treatment of this
connection between the prediction and its loss.

As a concrete terminal decision, suppose that selecting an arm also
requires a prediction $x\in[0,1]$ of $\sqrt\theta$, with payoff
$\theta-a(\sqrt\theta-x)^2$, where $0\le a\le1$.
Put $\kappa_t=\E[\sqrt\theta\mid\F_t]$ and
$\sigma_t^2=\Var(\sqrt\theta\mid\F_t)=\mu_t-\kappa_t^2$.
Then
\[
 \E[\theta-a(\sqrt\theta-x)^2\mid\F_t]
   =\mu_t-a\bigl[\sigma_t^2+(\kappa_t-x)^2\bigr].
\]
The optimal prediction is $x=\kappa_t$, and its optimized value is
\begin{equation}\label{eq:risk-score}
 h_t=\mu_t-a\Var(\sqrt\theta\mid\F_t)
    =(1-a)\mu_t+a\kappa_t^2\in[0,1].
\end{equation}
This optimized conditional expectation is the selectable payoff.
The penalty is the Bayes prediction risk for the latent quantity
$\sqrt\theta$. Predicting $\theta$ itself under squared error gives
the analogous value $\mu_t-av_t$. These variances describe posterior
uncertainty about the latent mean, rather than the variability of
individual future observations.

At $a=1$, $h_t=\kappa_t^2$ is also the certainty equivalent for
square-root utility: its square root is the posterior expected utility
$\kappa_t$ \citep{pratt1964}. Using this certainty equivalent differs
from using $\kappa_t$ itself. For example, an arm known to have
$\theta=1/2$ and an arm equally likely to have $\theta=0$ or $1$
have the same posterior mean, but their values in
\eqref{eq:risk-score} are $1/2$ and $1/2-a/4$, respectively.

\subsection{Value of information and policy execution}

The relevant martingale properties are consequences of classical
value-of-information arguments \citep{blackwell1953,degroot1962}.
For either squared-error prediction problem above, let
$\mathcal B_t$ denote its minimum conditional expected loss. A
prediction chosen at time $t$ remains available after the next
observation; reoptimizing can only improve on that prediction. Hence
\[
 \E[\mathcal B_{t+1}\mid\F_t]\le\mathcal B_t.
\]
Thus posterior Bayes risk is a supermartingale. A posterior expected
utility for a fixed terminal action is a martingale; optimizing over
the available terminal actions gives a submartingale. We record the
following particular cases to identify which execution argument is
applicable.

Conditional total variance gives
\begin{equation}\label{eq:variance-drift}
 \E[v_{t+1}\mid\F_t]
   =v_t-\E[(\mu_{t+1}-\mu_t)^2\mid\F_t].
\end{equation}
Consequently:
\begin{itemize}
\item $\E[g(\theta)\mid\F_t]$ is a martingale for every bounded
nonnegative $g$. This includes posterior threshold probabilities,
$\kappa_t$, and $\E[\theta^2\mid\F_t]=\mu_t^2+v_t$.
Conditional Jensen makes a nonnegative concave function of $\mu_t$
a supermartingale, and a nonnegative convex function a submartingale.
\item For $a\ge0$, $\mu_t+av_t$ is a nonnegative supermartingale,
as is $\mu_t+a\sqrt{v_t}$, since
$\E[\sqrt{v_{t+1}}\mid\F_t]
\le\sqrt{\E[v_{t+1}\mid\F_t]}\le\sqrt{v_t}$.
\item $\mu_t-av_t$ is a submartingale. For $0\le a\le1$,
$v_t\le\mu_t(1-\mu_t)$ gives
$\mu_t-av_t\ge(1-a)\mu_t+a\mu_t^2\ge0$.
\item For $h_t$ in \eqref{eq:risk-score}, the martingale property of
$\kappa_t$ gives
\[
 \E[h_{t+1}\mid\F_t]-h_t
   =a\,\E[(\kappa_{t+1}-\kappa_t)^2\mid\F_t]\ge0.
\]
This is the value of improving the terminal prediction, with the arm
held fixed.
\end{itemize}

\paragraph{When interruption is allowed.}
If the terminal score is nonnegative and a supermartingale, its
current value bounds the expected selected score of an interrupted
local policy. Theorem~\ref{thm:231} therefore applies with $r_u$
replaced by that score. This covers fixed-utility posteriors, including
the satisficing probabilities above, nonnegative concave functions of
$\mu_t$, and the two positive variance adjustments just listed.
Supermartingale selectable utilities already occur in the budgeted
learning framework of \citet[Section~6]{GM08}.

\paragraph{When the local policy must finish.}
For submartingale values such as $h_t$ or $\mu_t-av_t$, the current
value need not bound what a continuation would obtain. The
interruption step of Theorem~\ref{thm:231} is then unavailable.
Theorem~\ref{thm:short} instead completes every admitted local policy
and requires only nonnegative terminal payoffs. To use it, we need
short local policies whose selected scores approximate those of the
full policies. The next subsection supplies this comparison for
expected utilities $\E[g(\theta)\mid\F_t]$ with uniformly continuous
$g$, Lipschitz functions of $\mu_t$, $\mu_t\pm av_t$, and the
prediction score $h_t$.

\subsection{Loss from stopping a local experiment after \texorpdfstring{$k$}{k} observations}

Use exactly the operation from Section~\ref{sec:truncation}:
follow a local policy $P$ until it terminates or reaches depth $k$;
if still active, select with its conditional eventual-selection
probability $\chi_k$ from \eqref{eq:coin}, and otherwise quit.
The resulting policy $P^{[k]}$ preserves selection probability and
uses no more plays; only its terminal score can change. As in
Corollary~\ref{lem:depth-k}, we charge that change to the selection
probability, using the empirical mean $\widehat\theta_k$ of the first
$k$ observations as an $\F_k$-measurable comparator. Since a posterior
mean minimizes conditional mean squared error, for every stopping time
$\tau\ge k$,
\begin{equation}\label{eq:bayes-risk-compare}
 v_\tau\le\E[(\theta-\widehat\theta_k)^2\mid\F_\tau],\qquad
 \sigma_k^2\le\E[(\sqrt\theta-\sqrt{\widehat\theta_k})^2\mid\F_k].
\end{equation}

Let $C$ be the event that $P$ is still active at the cutoff, and
$S=C\cap\{P\text{ eventually selects}\}$. Set
$\tau=\max\{\tau(P),k\}$, so that $k\le\tau\le T$ and it agrees
with the original stopping time on $S$. For a local score process
$\psi_t$, write
$\widetilde R_\psi(P)=\E[\psi_{\tau(P)}\ind\{P\text{ selects}\}]$.
For $g:[0,1]\to[0,1]$, let
$\omega_g(z)=\sup\{|g(x)-g(y)|:x,y\in[0,1],\ |x-y|\le z\}$ be its
modulus of continuity.

\begin{lemma}\label{lem:score-trunc}
Let $(\psi_t)_{t=0}^T$ be an integrable process adapted to $(\F_t)$,
and let $1\le k\le T$. Then
\[
 \widetilde R_\psi(P)-\widetilde R_\psi(P^{[k]})
   =\E[(\psi_\tau-\psi_k)\ind_S],
\]
while $I(P^{[k]})=I(P)$ and $W(P^{[k]})\le W(P)$.
For the bounded conditionally iid observations considered here and
every $0<z\le1$,
\begin{align}
 \widetilde R_{\E[g(\theta)\mid\cdot]}(P^{[k]})
   &\ge\widetilde R_{\E[g(\theta)\mid\cdot]}(P)
       -2\omega_g(z)I(P)-2e^{-2kz^2},
   \label{eq:utility-trunc}\\
 \widetilde R_{f(\mu)}(P^{[k]})
   &\ge\widetilde R_{f(\mu)}(P)-2KzI(P)-4Ke^{-2kz^2},
   \label{eq:lipschitz-trunc}\\
 \widetilde R_{\mu\pm av}(P^{[k]})
   &\ge\widetilde R_{\mu\pm av}(P)-(2z+az^2)I(P)-2(1+a)e^{-2kz^2},
   \label{eq:variance-trunc}\\
 \widetilde R_h(P^{[k]})
   &\ge\widetilde R_h(P)-(2+a)zI(P)-2(1+a)e^{-2kz^2},
   \label{eq:risk-score-trunc}
\end{align}
where $g$ maps $[0,1]$ to $[0,1]$, $f$ is $K$-Lipschitz, $a\ge0$ in
\eqref{eq:variance-trunc}, and $h$ is \eqref{eq:risk-score} with
$0\le a\le1$.
\end{lemma}
\begin{proof}
The policies agree off $C$. On $C$, the new selection coin is
independent of the future conditional on $\F_k$, and
$Z=\chi_k\ind_C=\E[\ind_S\mid\F_k]$. The replacement obtains
$\E[\psi_k Z]=\E[\psi_k\ind_S]$, whereas $P$ obtains
$\E[\psi_\tau\ind_S]$ on $S$. This proves the identity;
selection probability and work satisfy \eqref{eq:match}.

For the bounds, let $G=\{|\widehat\theta_k-\theta|\le z\}$.
Conditional on $D$, Hoeffding's inequality gives
$\Prb(G^c)\le2e^{-2kz^2}$. Every $\F_k$-measurable $X$ satisfies
$\E[X\ind_S]=\E[XZ]$, and $\E Z=\Prb(S)\le I(P)$. The event $S$ is
determined by time $\tau$.

\emph{Expected utility.} For $\psi_t=\E[g(\theta)\mid\F_t]$, the tower
property gives $\E[\psi_\tau\ind_S]=\E[g(\theta)\ind_S]$ and
$\E[\psi_k\ind_S]=\E[g(\theta)Z]$. Since
$\E[g(\widehat\theta_k)(\ind_S-Z)]=0$,
\[
 \E[(\psi_\tau-\psi_k)\ind_S]
   =\E\bigl[(g(\theta)-g(\widehat\theta_k))(\ind_S-Z)\bigr].
\]
On $G$ the integrand is at most $\omega_g(z)(\ind_S+Z)$ in absolute
value, and off $G$ at most one. This gives \eqref{eq:utility-trunc}.
With $g(\theta)=\theta$, $\psi_t=\mu_t$ and $\omega_g(z)=z$; this is
the bound in the proof of Corollary~\ref{prop:additive-trunc}.

\emph{Lipschitz functions of the mean.} Since
$|\mu_\tau-\widehat\theta_k|\le\E[|\theta-\widehat\theta_k|\mid\F_\tau]$
and $|\mu_k-\widehat\theta_k|\le\E[|\theta-\widehat\theta_k|\mid\F_k]$,
\[
 \E[|\mu_\tau-\mu_k|\ind_S]
 \le\E[|\theta-\widehat\theta_k|\ind_S]
   +\E[|\theta-\widehat\theta_k|Z]
 \le2z\Prb(S)+2\Prb(G^c),
\]
which gives \eqref{eq:lipschitz-trunc}.

\emph{Variance adjustments.} The increment of $\mu+av$ is at most
$(\mu_\tau-\mu_k)+av_\tau$, and that of $\mu-av$ is at most
$(\mu_\tau-\mu_k)+av_k$. The mean part is \eqref{eq:utility-trunc}
with $g(\theta)=\theta$. By \eqref{eq:bayes-risk-compare},
$\E[v_\tau\ind_S]\le\E[(\theta-\widehat\theta_k)^2\ind_S]
\le z^2\Prb(S)+\Prb(G^c)$, and likewise
$\E[v_k\ind_S]=\E[v_kZ]\le\E[(\theta-\widehat\theta_k)^2Z]
\le z^2\Prb(S)+\Prb(G^c)$.

\emph{Prediction score.} The increment of $h$ is
$(\mu_\tau-\mu_k)-a\sigma_\tau^2+a\sigma_k^2
\le(\mu_\tau-\mu_k)+a\sigma_k^2$. By \eqref{eq:bayes-risk-compare}
and $(\sqrt x-\sqrt y)^2\le|x-y|$,
\[
 \E[\sigma_k^2\ind_S]=\E[\sigma_k^2Z]
 \le\E\bigl[(\sqrt\theta-\sqrt{\widehat\theta_k})^2Z\bigr]
 \le z\Prb(S)+\Prb(G^c).
\]
Adding the mean part gives \eqref{eq:risk-score-trunc}.
\end{proof}

Summed over arms with $\sum_iI(P_i)\le1$, the expected-utility bound
\eqref{eq:utility-trunc} gives aggregate loss at most
$2\omega_g(z)+2ne^{-2kz^2}$. Choosing $z$ with $\omega_g(z)\le\epsilon/4$
and then $k\ge(2z^2)^{-1}\log(4n/\epsilon)$ makes it at most
$\epsilon$. The required depth depends on the modulus of $g$: for
$g(x)=\sqrt x$, $\omega_g(z)=\sqrt z$, so $z=\epsilon^2/16$ and $k$ is
of order $\epsilon^{-4}\log(n/\epsilon)$. In
\eqref{eq:lipschitz-trunc}--\eqref{eq:risk-score-trunc} the loss is a
constant times $z$ plus $O(ne^{-2kz^2})$, so $z$ of order $\epsilon$
and $k$ of order $\epsilon^{-2}\log(n/\epsilon)$ suffice. In every case
the cutoff depends on $n$ only through $\log n$, as in
Section~\ref{sec:truncation}. For the satisficing utility
$\ind\{\theta\ge\theta_0\}$ with $0<\theta_0\le1$, $\omega_g(z)=1$ for
every $z>0$, and \eqref{eq:utility-trunc} gives no vanishing bound;
that criterion is a martingale and is covered by
Theorem~\ref{thm:231}.

\subsection{A corollary for joint selection and prediction}

We now combine Lemma~\ref{lem:score-trunc} with
Theorem~\ref{thm:short} for the prediction score. Allocate at most $T$
observations, then select an arm $J$ and report a prediction $x$ to
maximize
\[
 \E\bigl[\theta_J-a(\sqrt{\theta_J}-x)^2\bigr].
\]
At the final history $H$, the best prediction for a chosen arm is
$x=\E[\sqrt{\theta_J}\mid H]$. Optimizing also over $J$ makes this
objective exactly $\E[\max_i h_i(H)]$.

Let $\mathrm{LP}_{2,h}$ be the two-constraint relaxation
\eqref{eq:two} with terminal coefficient $h_u$ in place of $r_u$.
It upper-bounds the optimal value of this joint decision problem.

\begin{corollary}[Joint selection and prediction]
\label{cor:prediction-score}
Let $0\le a\le1$, $0<\epsilon<1$, $0<\delta<\rh$, and
\[
 k=\min\bigl\{T,\ \lceil18\epsilon^{-2}\log(8n/\epsilon)\rceil\bigr\}.
\]
If $\gamma<1$ in \eqref{eq:slack}, solving the depth-$k$ relaxation
with terminal values $h_u$, then applying
Theorem~\ref{thm:short} and using the Bayes prediction for the
selected arm, gives
\begin{equation}\label{eq:risk-score-guarantee}
 \ALG_h\ge(\rh-\delta)(1-\gamma)
       \big[\mathrm{LP}_{2,h}-\epsilon\big]^+.
\end{equation}
If $\mathrm{LP}_{2,h}$ is bounded below by a positive constant and
$\log n=o(T)$, choosing $\epsilon$ and $\delta$ that vanish
sufficiently slowly gives a limiting approximation factor
$\rh=1-e^{-1}$.
\end{corollary}
\begin{proof}
Apply the selection-preserving stopping operation to each optimal
local policy of $\mathrm{LP}_{2,h}$. If $k=T$, no policy is shortened.
Otherwise set $z=\epsilon/(2(2+a))$ in \eqref{eq:risk-score-trunc}
and sum with $\sum_iI(P_i)\le1$. The total loss is at most
$(2+a)z+2(1+a)ne^{-2kz^2}=\epsilon/2+2(1+a)ne^{-2kz^2}$.
Since $2(2+a)^2\le18$, we have $2kz^2\ge k\epsilon^2/18\ge\log(8n/\epsilon)$,
so the second term is at most $4n\cdot\epsilon/(8n)=\epsilon/2$.
By Lemma~\ref{lem:score-trunc} the shortened policies satisfy both
resource constraints, and their total score is nonnegative because
$h_u\ge0$. The depth-$k$ LP therefore has value at least
$[\mathrm{LP}_{2,h}-\epsilon]^+$, and Theorem~\ref{thm:short},
which completes every admitted policy, gives
\eqref{eq:risk-score-guarantee}.
For the limit, take $\delta=\epsilon$. Then
$k\log(1/\delta)=O\bigl(\epsilon^{-2}(\log n+\log(1/\epsilon))\log(1/\epsilon)\bigr)$,
which is $o(T)$ when $\log n=o(T)$ and $\epsilon\to0$ slowly enough.
Hence $\gamma\to0$, and dividing \eqref{eq:risk-score-guarantee} by
$\mathrm{LP}_{2,h}$ proves the limit.
\end{proof}

The same composition applies to the other criteria of
Lemma~\ref{lem:score-trunc} with nonnegative scores. For nonnegative
Lipschitz functions of $\mu_t$ and for $\mu_t\pm av_t$, with
$0\le a\le1$ for the minus sign, only the constants in the choice of
$z$ change. For an expected utility with uniformly continuous $g$,
choose $z$ with $\omega_g(z)\le\epsilon/4$ and
$k\ge(2z^2)^{-1}\log(4n/\epsilon)$, as after
Lemma~\ref{lem:score-trunc}. The depth then grows faster than
$\epsilon^{-2}$ when $\omega_g(z)$ is not $O(z)$, but it still depends
on $n$ only through $\log n$. For a fixed such $g$, the limiting factor
$\rh$ therefore holds when the corresponding LP value is bounded below
by a positive constant and $\log n=o(T)$. The bound is additive in
$\epsilon$; we do not pursue a relative version, as in
Lemma~\ref{lem:trunc}, for these criteria.

\subsubsection*{AI Use Statement}
Multiple generative AI tools were used in this work for improving the writing and generating tight examples. We have reviewed all AI-assisted work.

\setcitestyle{numbers,square}
\renewcommand{\natexlab}[1]{#1}
\bibliographystyle{plainnat}
\bibliography{refs}

\end{document}